\documentclass[a4paper,fleqn]{cas-sc}
\usepackage[numbers,sort&compress]{natbib}
\usepackage{amssymb}
\usepackage{amsthm} 
\usepackage{amsmath,amsfonts}
\usepackage{algorithmic}
\usepackage{algorithm}
\usepackage{array}
\usepackage[caption=false,font=footnotesize,labelfont=rm,textfont=rm]{subfig}
\usepackage{textcomp}
\usepackage{url}
\usepackage{verbatim}
\usepackage{graphicx}
\usepackage{pifont} 
\usepackage{booktabs}
\usepackage{multirow}
\usepackage{bm}
\usepackage{xcolor}

\def\tsc#1{\csdef{#1}{\textsc{\lowercase{#1}}\xspace}}
\tsc{WGM}
\tsc{QE}
\tsc{EP}
\tsc{PMS}
\tsc{BEC}
\tsc{DE}

\begin{document}
\let\WriteBookmarks\relax
\def\floatpagepagefraction{1}
\def\textpagefraction{.001}

\shorttitle{ }

\shortauthors{Xuelin Xie et~al.}

\title [mode = title]{Spatial-Spectral Adaptive Fidelity and Noise Prior Reduction Guided Hyperspectral Image Denoising}
\tnotemark[1]

\tnotetext[1]{This work was supported in part by the Science and Technology Development Fund, Macao S.A.R. under Grant 0089/2024/AGJ and Grant 0012/2025/RIA1, in part by the University of Macau and the University of Macau Development Foundation under Grant MYRG-GRG2023-00106-FST-UMDF, in part by the National Natural Science Foundation of China under Grant 12371424, Grant U24A2002, and Grant 12561160122, and in part by the Natural Science Foundation of Hubei Province, China under Grant 2024AFE006.}

\author[1,2]{Xuelin Xie}[orcid=0000-0003-1135-4659]
\ead{xl.xie@whu.edu.cn}

\author[1,3,4]{Xiliang Lu}[orcid=0000-0002-7592-5994]
\ead{xllv.math@whu.edu.cn}

\author[2]{Zhengshan Wang}[orcid=0000-0001-7371-7465]
\ead{yc27436@umac.mo}

\author[2]{Yang Zhang}[orcid=0009-0008-1692-7579]
\ead{yc47958@um.edu.mo}

\author[2]{Long Chen}[orcid=0000-0003-0184-5446]
\ead{longchen@um.edu.mo}
\cormark[1]

\affiliation[1]{organization={School of Mathematics and Statistics, Wuhan University},
                city={Wuhan},
                postcode={430072}, 
                country={China}}
\affiliation[2]{organization={Department of Computer and Information Science, Faculty of Science and Technology, University of Macau},
                postcode={999078},
                city={Macau},
                country={China}}
\affiliation[3]{organization={Hubei Center for Applied Mathematics, Wuhan University},
                city={Wuhan},
                postcode={430072}, 
                country={China}}
\affiliation[4]{organization={Hubei Key Laboratory of Computational Science, Wuhan University},
                city={Wuhan},
                postcode={430072}, 
                country={China}}

\cortext[cor1]{Corresponding author}

\begin{abstract}
The core challenge of hyperspectral image (HSI) denoising is striking the right balance between data fidelity and noise prior modeling. Most existing methods place too much emphasis on the intrinsic priors of the image while overlooking diverse noise assumptions and the dynamic trade-off between fidelity and priors. To address these issues, we propose a denoising framework that integrates noise prior reduction and a spatial–spectral adaptive fidelity term. This framework considers comprehensive noise priors with fewer parameters and introduces an adaptive weight tensor to dynamically balance the fidelity and prior regularization terms. Within this framework, we further develop a fast and robust pixel-wise model combined with the representative coefficient total variation regularizer to accurately remove mixed noise in HSIs. The proposed method not only efficiently handles various types of noise but also accurately captures the spectral low-rank structure and local smoothness of HSIs. An efficient optimization algorithm based on the alternating direction method of multipliers is designed to ensure stable and fast convergence. Extensive experiments on simulated and real-world datasets demonstrate that the proposed model achieves superior denoising performance while maintaining competitive computational efficiency. The source code for the proposed FRHD method is publicly available at \url{https://github.com/xuelin-xie/FRHD}.
\end{abstract}


\maketitle
\section{Introduction}
\subsection{Background}
Hyperspectral images (HSIs) capture the spatial and spectral information of observed scenes, thereby enabling accurate characterization of their physical and chemical properties through rich spectral signatures~\cite{Zhang2022Cooperated}. This distinctive capability has positioned Hyperspectral imaging as an indispensable technology across diverse fields, including remote sensing \cite{rasti2021image,li2025deep}, medical diagnostics \cite{dremin2021skin,wang2023deep}, face recognition \cite{zhang2021hyperspectral}, and vegetation analysis \cite{okujeni2021multi}. Nevertheless, the HSI acquisition process is inherently affected by uncertainties stemming from sensor limitations, photon noise, illumination fluctuations, and atmospheric interference, which introduce severe and varied noise contamination \cite{peng2021low}. Such contamination severely degrades data integrity and complicates subsequent analysis, rendering effective denoising and high-fidelity data recovery critical challenges in HSI processing.

In recent years, a variety of advanced algorithms for HSI denoising have been proposed, driving substantial progress \cite{He2022NGmeet, Chen2024flex}. These methods are primarily categorized into three paradigms: model-based, data-driven, and hybrid approaches that combine both model-based and data-driven strategies. As the current mainstream approach, data-driven methods adaptively optimize trainable parameters based on external data and available labels within a deep learning framework \cite{zhang2024hyperspectral}. Leveraging their strong self-learning capabilities, various deep learning approaches have been successfully applied to tasks such as denoising, restoration, reconstruction, and high-resolution enhancement of HSIs. 

Data-driven architectures such as convolutional neural networks (CNNs) \cite{Maffei2020, zhuang2024eigen}, attention mechanisms \cite{zhangq2024}, and Transformers \cite{li2023spectral} have established new paradigms for HSI denoising. For instance, HSI-DeNet \cite{chang2018hsi}, a CNN-based approach, exploits spectral coherence by fusing adjacent spectral bands via multi-channel 2D convolutions. Building upon attention mechanisms, the 3D Attention Denoising Network \cite{shi2021hyperspectral} models global spatial–spectral dependencies, while Transformer-based frameworks leverage long-range correlations to complement the local modeling capacity of CNNs \cite{li2023spectral}. Moreover, to incorporate physical constraints and extract intrinsic deep priors, \cite{xiong2023deep} integrates multiple image priors into CNN architectures. Recent advances that combine deep learning with model-driven strategies have further propelled the field forward \cite{Zhang2022Cooperated, Litci2024}. Xiong et al. proposed SMDS-Net~\cite{xiong2022smds}, which integrates spectral low-rankness and spatial sparsity through a network unfolding framework. Zhuang and Ng introduced FastHyMix~\cite{zhuang2023fasthymix}, which incorporates a pre-trained FFDNet denoiser with spectral low-rank constraints. Subsequently, Chen et al. developed Flex-DLD~\cite{Chen2024flex}, a flexible deep low-rank decomposition method that further embeds traditional prior knowledge into the learning architecture. By employing composite loss functions and hybrid designs, these methods strike an effective balance between data-driven adaptability and prior-guided regularization, thereby enhancing both robustness and reconstruction accuracy. Despite these successes, learning-based methods often require extensive training data and suffer from limited interpretability and generalization due to their inherent black-box nature.

From an objective standpoint, when trained on sufficiently large and representative datasets, learning-based methods may surpass model-based approaches in terms of precision and accuracy  \cite{xie2022poplar}. Nevertheless, model-based methods still continue to attract considerable research interest owing to their solid theoretical foundations and strong interpretability. Model-driven approaches primarily exploit the intrinsic properties of HSI to mitigate noise. Common techniques employed in these works include Total Variation (TV) \cite{yuan2012hyperspectral, tao2025alternative}, non-local (NL) similarity \cite{Dabov2007, sarkar2021non}, sparse representation (SR) \cite{zhao2025tensor}, and low-rank (LR) matrix/tensor decomposition \cite{chang2020weighted,zhu2024projection}. Representative TV-based methods mitigate noise by minimizing the total variation of the image, which effectively preserves edge details while smoothing homogeneous regions. On the other hand, the intrinsic low-rank property of HSI has prompted researchers to apply low-rank matrix or tensor decomposition to HSI data, followed by refinement of the decomposed components using additional methods for enhanced denoising and reconstruction. Numerous studies have explored combinations~\cite{peng2022fast, chen2022hyperspectral, chen2023TPTV, peng2024stable} of different model-driven techniques, effectively leveraging their complementary advantages.

\subsection{Motivation and our contributions}
Most existing HSI denoising methods predominantly focus on incorporating image priors to constrain the fidelity term or on developing specialized models for specific noise types. Methods that apply different constraints to the fidelity term often assume Gaussian or sparse noise, yet typically neglect noise priors such as stripe noise, which possess distinct directional and structural characteristics. Conversely, approaches tailored to specific noise types rely on single noise priors, limiting their effectiveness in mixed noise scenarios. Overall, current denoising frameworks suffer from two major issues. On the one hand, models that comprehensively incorporate various noise priors often lead to excessive parameters and optimization difficulties. On the other hand, the trade-off between the fidelity term and the regularization prior is typically handled only through parameter tuning, which may not yield the optimal solution.

Motivated by the challenges associated with the aforementioned issues, we decompose the noise prior into three distinct components: uniform, sparse, and directional structural noise. Building on this, we establish a noise prior reduction framework for HSI denoising to reduce both model parameters and optimization complexity. To more effectively balance the fidelity and regularization terms, we break the conventional limitation of a fixed fidelity term by introducing a pixel-wise spatial–spectral adaptive fidelity term. In addition, we integrate the Representative Coefficient Total Variation (RCTV) regularization to construct the proposed \textbf{F}ast and \textbf{R}obust \textbf{H}ybrid \textbf{D}enoising (FRHD) model for HSIs. Distinct from traditional model-based methods, the FRHD model incorporates comprehensive noise priors in a more concise parameterization and optimization framework, and adaptively balances the fidelity and regularization terms. Experimental results demonstrate that the proposed method achieves superior performance in both denoising quality and computational efficiency. 

In summary, the contributions of this work can be articulated as follows:

\begin{itemize}  
\item We reformulate the noise assumption by considering diverse attributes of mixed noise. From a theoretical perspective, we rigorously demonstrate that, with appropriate parameter tuning, the solution variables of the optimization problem in Equation (\ref{EP1}) are equivalent to those in Equation (\ref{NRHSIframework}) within the ADMM framework. This achieves an elegant simplification of complex noise priors and establishes a noise prior reduction framework for HSI denoising.

\item We develop a spatial–spectral noise-adaptive fidelity term that dynamically adjusts according to the pixel-wise noise level. The fidelity incorporates a learnable weight tensor, which is iteratively updated to reflect noise evolution, thereby enhancing the model’s denoising robustness and adaptivity. Furthermore, by integrating the RCTV regularization, we further propose the FRHD model.

\item An efficient optimization algorithm based on the alternating direction method of multipliers (ADMM) framework is proposed, with rigorous derivations and convergence analysis provided for the overall algorithm. Extensive experiments on simulated and real-world datasets validate the superior denoising performance and competitive efficiency of the proposed FRHD method.
\end{itemize}

The remainder of this paper is structured as follows. Section \ref{sect2} reviews related works pertinent to this study. Section \ref{sect3} presents the proposed noise prior reduction framework and the FRHD method. Section \ref{sect4} presents the optimization algorithm employed to solve the model. In Section \ref{sect5}, we report extensive experimental results and provide a comprehensive analysis. Finally, Section \ref{sect6} concludes this work.

\section{Notation and Preliminaries}   \label{sect2}
\subsection{Notations.}
In this paper, scalars are denoted by regular letters, matrices by bold capitals, and tensors by calligraphic capitals. To facilitate a clearer understanding of the model and algorithm, key notations are summarized in Table~\ref{Notation_def}. In particular, the mode-3 unfolding operator, $\mathrm{unfold}_3(\cdot)$, maps a tensor $\mathcal{X} \in \mathbb{R}^{M \times N \times B}$ to a matrix $\mathbf{X} \in \mathbb{R}^{MN \times B}$ by vectorizing each spatial slice along the spectral dimension, and its inverse, $\mathrm{fold}_3(\cdot)$, reshapes the matrix back into the original tensor. In this paper, these operators are used consistently.

\begin{table}[h!]
\fontfamily{ptm}\selectfont 
\caption{The main symbols and their definitions.}  \label{Notation_def}
\renewcommand{\arraystretch}{0.9}
\centering
\begin{tabular}{|l|l|} 
\hline
\textbf{Notation} & \textbf{Definition} \\ \hline
$\mathbf{Y}/\mathcal{Y} \hspace{1.6em}$ & Noisy HSI matrix/tensor. \\ \hline
$\mathbf{X}/\mathcal{X}$ & Restored HSI matrix/tensor. \\ \hline
$\mathbf{N}/\mathcal{N}$ & Noise matrix/tensor. \\ \hline
$\mathbf{S}/\mathcal{S}$ & Sparse noise matrix/tensor. \\ \hline
$\mathbf{D}/\mathcal{D}$ & Deadline or stripe noise/tensor. \\ \hline
$\mathbf{W}/\mathcal{W}$ & Weight matrix/tensor. \\ \hline
$\mathbf{U}/\mathcal{U}$  & Representative coefficients matrix/tensor. \hspace{3em}\\ \hline
$\bm{\mathfrak{T}}/\mathcal{T} $ & All-ones matrix/tensor. \\ \hline
$\mathbf{V}$  & Orthogonal matrix. \\ \hline
$\mathbf{H}_i$ & Auxiliary variable of ADMM.  \\ \hline
$\odot$ & Hadamard product. \\ \hline
$\mathbb{R}$ & Set of real numbers. \\ \hline
$\mathcal{R}(\mathcal{X})$ & Image prior. \\ \hline
$\mathcal{P}(\mathcal{N})$ & Noise prior. \\ \hline
$\mathcal{J}(\mathcal{W})$ & Weight prior. \\ \hline
\end{tabular}
\end{table}

\subsection{Preliminaries of HSI Denoising.}
The traditional HSI denoising task aims to recover a clean image from noise-corrupted observations \cite{su2023hyperspectral}, satisfying the following equation:
\begin{equation}  \label{HSI_denoising}
\mathcal{Y} = \mathcal{X} + \mathcal{N},
\end{equation}
where $\mathcal{Y} \in \mathbb{R}^{M \times N \times B}$ denotes the observed HSI, $\mathcal{X} \in \mathbb{R}^{M \times N \times B}$ is the clean HSI, and $\mathcal{N} \in \mathbb{R}^{M \times N \times B}$ represents the noise tensor. The size $M$ and $N$ correspond to the spatial resolution, and $B$ represents the number of spectral bands.

The primary objective of HSI denoising is to preserve both the spatial and spectral structures of the image while effectively removing noise, which can be formulated as the following optimization problem:
\begin{equation}   \label{HSI_frame}
\min_{\mathcal{X}} \| \mathcal{Y} - \mathcal{X} \|_F^2 + \lambda \mathcal{R}(\mathcal{X}), 
\end{equation} 
where $\mathcal{R}(\mathcal{X})$ guides the denoising process to retain the intrinsic characteristics of the image, and $\lambda$ serves as the regularization parameter.

Since the noise tensor $\mathcal{N}$ primarily consists of uniform Gaussian noise and other forms of sparse noise, the problem in (\ref{HSI_denoising}) is further decomposed into a more specific formulation:
\begin{equation}  \label{HSI_newdenoising}
\mathcal{Y} = \mathcal{X} + \mathcal{G}+ \mathcal{S},
\end{equation}
where $\mathcal{G}$ denotes the Gaussian noise, and $\mathcal{S}$ represents the other system-specific sparse corruptions.
 
\subsection{TV and Low-rank RCTV regularization.}
In the context of image denoising, TV regularization effectively encodes prior knowledge of local smoothness. For a grayscale image $\mathbf{X}$ of size $M \times N$, the TV norm \cite{yuan2012hyperspectral} is commonly defined as:
\begin{align}
& \|\mathbf{X}\|_{\text{TV}}  := \sum_{i=1}^{M-1} \sum_{j=1}^{N-1} \big( |x_{i,j} - x_{i+1,j}|  + |x_{i,j} - x_{i,j+1}| \big) + \! \sum_{i=1}^{M-1} |x_{i,N} - x_{i+1,N}| + \sum_{j=1}^{N-1} |x_{M,j} - x_{M,j+1}|.
\end{align}

In general, HSI data is represented as a 3D tensor $\mathcal{X}$ comprising multiple bands. TV regularization can be applied to each band individually or by first unfolding $\mathcal{X}$ into a matrix and subsequently applying TV regularization. However, both strategies face limitations in computational efficiency and accuracy, as leveraging all spectral bands does not inherently yield optimal reconstruction performance.

\begin{figure}[!t]   
\centering
\includegraphics[width=4.15in]{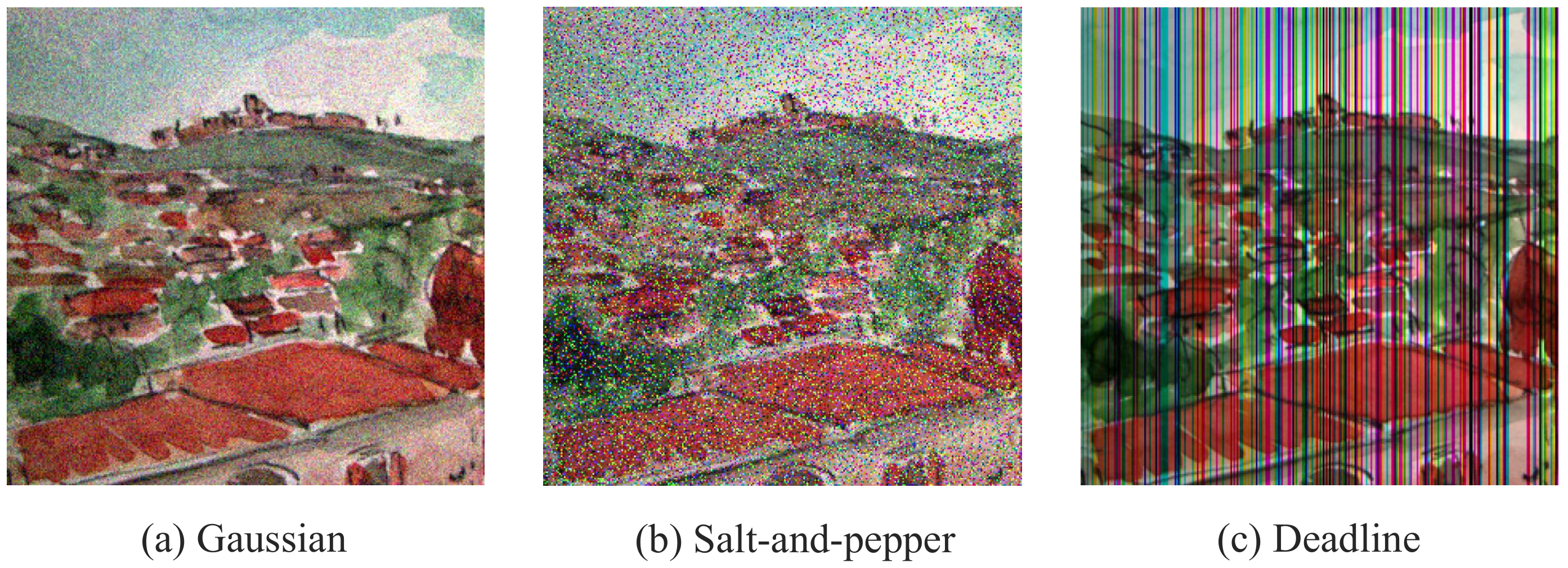}
\caption{Description of the three different types of noise. }
\label{noisy_demo}
\end{figure}

\begin{table*}[!t]
\fontfamily{ptm}\selectfont 
\caption{Comparison of Methodological Properties.}   \label{method_prior}
\centering
\renewcommand{\arraystretch}{1.25}
\resizebox{0.97\textwidth}{!}{
\begin{tabular}{lcccccccccccc}
\hline 
\textbf{Property} & 
\shortstack{SSAHTV \\ {\cite{yuan2012hyperspectral}}} & 
\shortstack{LRTV \\ {\cite{he2015total}}} & 
\shortstack{FastHyDe \\ {\cite{zhuang2018fast}}} & 
\shortstack{NGmeet \\ {\cite{He2022NGmeet}}} & 
\shortstack{LSSTV \\ {\cite{wang2018low}}} & 
\shortstack{WLRTR \\ {\cite{chang2020weighted}}} & 
\shortstack{WNLRATV \\ {\cite{chen2022hyperspectral}}} & 
\shortstack{RCTV \\ {\cite{peng2022fast}}} & 
\shortstack{TPTV \\ {\cite{chen2023TPTV}}} & 
\shortstack{CTV-SPCP \\ {\cite{peng2024stable}}} & 
\shortstack{FRHD \\ Ours} \\
\hline
Spectral Low-Rank  & \ding{55} & \ding{51} & \ding{51} & \ding{51} & \ding{51} & \ding{51} & \ding{51} & \ding{51} & \ding{51} & \ding{51} & \ding{51} \\
Gauss Noise Prior & \ding{51} & \ding{51} & \ding{51} & \ding{51} & \ding{51} & \ding{51} & \ding{51} & \ding{51} & \ding{51} & \ding{51} & \ding{51} \\
Sparse $\ell_{1}$-norm Prior & \ding{55} & \ding{51} & \ding{51} & \ding{55} & \ding{51} & \ding{55} & \ding{51} & \ding{51} & \ding{51} & \ding{51} & \ding{51} \\
Structured $\ell_{2,1}$-norm Prior & \ding{55} & \ding{55} & \ding{55} & \ding{55} & \ding{55} & \ding{51} & \ding{55} & \ding{55} & \ding{55} & \ding{55} & \ding{51} \\
TV Regularization & \ding{51} & \ding{51} & \ding{55} & \ding{55} & \ding{51} & \ding{55} & \ding{51} & \ding{51} & \ding{51} & \ding{51} & \ding{51} \\
RCTV Regularization   & \ding{55} & \ding{55} & \ding{55} & \ding{55} & \ding{55} & \ding{55} & \ding{55} & \ding{51} & \ding{55} & \ding{55} & \ding{51} \\
Adaptive Pixel-wise fidelity  & \ding{55} & \ding{55} & \ding{55} & \ding{55} & \ding{55} & \ding{55} & \ding{55} & \ding{55} & \ding{55} & \ding{55} & \ding{51} \\
\hline
\end{tabular}
}
\end{table*}

Recently, inspired by low-rank subspace modeling, Peng et al.~\cite{peng2022fast} proposed the RCTV regularization, which decomposes the input tensor into subspaces and selectively enforces TV regularization on the most informative bands. By applying TV to the spectrally reduced HSI rather than the full HSI, this method achieves improved computational efficiency and denoising performance. The RCTV regularization is defined as:
\begin{equation}
\|\mathbf{X}\|_{\text{RCTV}} := \sum_{i=1}^R \|\mathcal{U}(:, :, i)\|_{\text{TV}},
\end{equation}
where $\mathcal{U} = \mathcal{X} \times_3 \mathbf{V}$, $\mathbf{V} \in \mathbb{R}^{B \times R}$ is an orthogonal matrix. The explicit mathematical formulation of RCTV regularization is as follows:
\begin{equation}
\|\mathbf{X}\|_{\text{RCTV}} := \|\nabla_1(\mathbf{U})\|_1 + \|\nabla_2(\mathbf{U})\|_1,
\end{equation}
where $\mathbf{X} \in \mathbb{R}^{MN \times B}$ is the matrix obtained after unfolding the HSI, $\mathbf{X} = \mathbf{U} \mathbf{V}^\top$ represents the low-rank decomposition, $\mathbf{U} \in \mathbb{R}^{MN \times B}$ is the matrix of representative coefficients.

\subsection{Overview and comparison of prior-based methods.}
To review the priors of model-based methods and highlight the methodological and conceptual differences, we conduct a comparative analysis between the proposed FRHD method and a set of classical and representative hyperspectral image denoising approaches, including SSAHTV~\cite{yuan2012hyperspectral}, LRTV~\cite{he2015total}, FastHyDe~\cite{zhuang2018fast}, NGmeet~\cite{He2022NGmeet}, LSSTV~\cite{wang2018low}, WLRTR~\cite{chang2020weighted}, WNLRATV~\cite{chen2022hyperspectral}, RCTV~\cite{peng2022fast}, TPTV~\cite{chen2023TPTV}, and CTV-SPCP~\cite{peng2024stable}.

Among these methods, SSAHTV, LRTV, FastHyDe, and NGmeet are early representative approaches, each embedding specific priors such as total variation regularization, spectral low-rank structure, explicit noise modeling, or non-local self-similarity. The remaining methods are relatively recent and incorporate more advanced priors or structural constraints. We systematically compare the underlying priors adopted by these methods with those employed in our FRHD model. 

{As shown in Table~\ref{method_prior}, our model incorporates more comprehensive and detailed prior information, and extends the TV and RCTV methods by integrating a pixel-adaptive weighting strategy}. In contrast to existing model-based or model-data hybrid approaches, FRHD introduces a novel multi-norm constrained noise modeling framework, where each norm is explicitly designed to characterize a specific type of noise. While recent model-data hybrid  driven methods such as FastHyMix and Flex-DLD leverage pre-trained denoising networks or self-supervised low-rank representations, they often lack explicit and structured noise prior modeling, which can lead to performance degradation under complex noise conditions. More importantly, FRHD proposes an innovative spatially-spectrally adaptive weighting tensor $\mathcal{W}$, optimized jointly with the image and noise components at the pixel level. This adaptive mechanism, absent in most existing methods, enables dynamic refinement during iterative optimization, thereby significantly enhancing noise suppression capability and reconstruction fidelity.

\begin{figure*}[!t]
\centering
\includegraphics[width=6.4in]{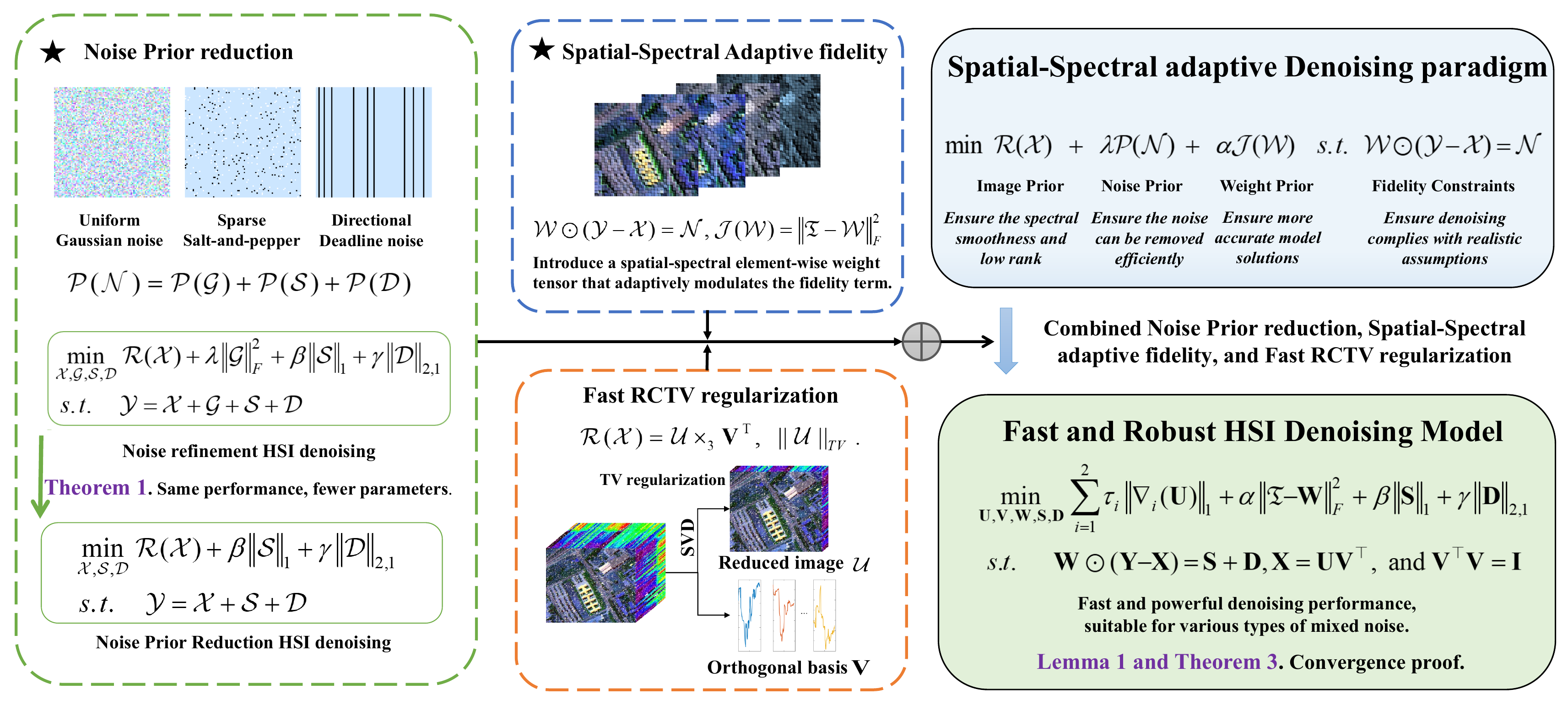}
\caption{Flowchart of the proposed framework and FRHD method.}
\label{Flowchart}
\end{figure*}

\section{Proposed framework and methodology}   \label{sect3}
\subsection{Noise Prior Reduction.}
Although existing denoising frameworks predominantly adhere to the structure of (\ref{HSI_frame}), they neglect the nuanced treatment of distinct noise priors. Accordingly, we deem that all image denoising tasks can be fundamentally framed as the following core problem:
\begin{align} \label{UniD}
\underset{\mathcal{X},\mathcal{N}}{\text{min}}\,  \hspace{-0.6em} \underbrace{\mathcal{R}(\mathcal{X})}_{\text{Image Prior}}\hspace{-0.5em} + \hspace{0.2em} \lambda \hspace{-0.5em} \underbrace{\mathcal{P}(\mathcal{N})}_{\text{Noise Prior}}\!, \hspace{0.4em} s.t. \hspace{0.4em}  \underbrace{\mathcal{Y=X+N}}_{\text{Fidelity Constraints}},
\end{align}
where $\mathcal{P}(\mathcal{N})$ is the noise prior governing and refining different types of noise constraints, and $\lambda$ is a trade-off parameter that balances the impact of different priors.

As shown in Fig.~\ref{noisy_demo}, noise in HSIs typically falls into three categories: Gaussian noise with uniform distribution, sparse impulse noise (e.g., salt-and-pepper), and structured directional noise such as stripes or deadlines. Among existing mixed noise removal methods, WNLRATV~\cite{chen2022hyperspectral}, RCTV~\cite{peng2022fast}, and TPTV~\cite{chen2023TPTV} successfully model mixed Gaussian and sparse noise components; however, their performance is constrained by an inability to capture the inherent directionality structure present in structured noise. BALMF~\cite{Xu2022hyper} introduces the banded asymmetric Laplace distribution to account for noise asymmetry, but its overly rigid prior often fails under complex real-world noise conditions. LR-TDTC~\cite{zhang2025hyperspectral} mitigates mixed noise via tensor decomposition and completion, but lacks explicit discrimination of noise types and incurs substantial computational cost (e.g., over 700 seconds for a $256 \times 256 \times 144$ HSI data), limiting practical applicability.

In contrast to previous methods, we consider decomposing HSI noise into three different types: \textit{i}) spatially homogeneous Gaussian noise; \textit{ii}) sparse, discrete noise (e.g., impulse or salt-and-pepper corruption); and \textit{iii}) structured stripe or deadline noise with directional sparsity. According to this categorization, we reformulate the fidelity term as:
\begin{align}
\mathcal{Y}=\mathcal{X+G+S+D},
\end{align}
where $\mathcal{G}$ denotes Gaussian noise, $\mathcal{D}$ represents directional stripe or deadline noise, and $\mathcal{S}$ corresponds to other forms of sparse noise. We further assert that each noise type should be regulated by a tailored constraint. Specifically, Gaussian noise is penalized using the smooth $\ell_2$-norm, sparse and discrete noise is handled via the $\ell_1$-norm~\cite{xie2024theta}, and directional noise is regularized with the $\ell_{2,1}$-norm. The $\ell_{2,1}$-norm~\cite{chang2020weighted} promotes structured sparsity by encouraging entire rows or columns to shrink toward zero.

Based on the aforementioned considerations, we refine the noise prior of the denoising problem (\ref{UniD}) and introduce the following definition:

\emph{\textbf{Definition 1} (Noise refinement HSI denoising framework)}: For a observed HSI tensor $\mathcal{Y} \in \mathbb{R}^{M \times N \times B}$, we define the Noise refinement HSI denoising framework as:
\begin{align} \label{EP1}
& \underset{\mathcal{X},\mathcal{G},\mathcal{S},\mathcal{D}}{\text{min}}\ \mathcal{R}(\mathcal{X})+\lambda \left\| \mathcal{G} \right\|_{F}^{2}+\beta {{\left\| \mathcal{S} \right\|}_{1}}+\gamma {{\left\| \mathcal{D} \right\|}_{2,1}} \quad s.t. \hspace{0.4em}  \mathcal{Y=X+G+S+D},
\end{align}
where $\lambda$, $\beta$, and $\gamma$ are trade-off parameters that balance the respective noise priors, $\left\| \mathcal{G} \right\|_{F}^{2}$, ${\left\| \mathcal{S} \right\|}_{1}$ and ${\left\| \mathcal{D} \right\|}_{2,1}$ representing regularization terms for uniform, sparse and directional noise, respectively.

\emph{\textbf{Definition 2} (Noise prior reduction HSI denoising framework)}: For a observed HSI tensor $\mathcal{Y} \in \mathbb{R}^{M \times N \times B}$, the Noise prior reduction HSI denoising framework is defined as:
\begin{align} \label{NRHSIframework}
& \underset{\mathcal{X},\mathcal{S},\mathcal{D}}{\text{min}}\  \mathcal{R}(\mathcal{X}) \!+\! \beta {{\left\| \mathcal{S} \right\|}_{1}} \!+\! \gamma {{\left\| \mathcal{D} \right\|}_{2,1}}  \hspace{0.2 em} s.t.  \mathcal{Y=X+S+D},
\end{align}
where the definitions of the variables remain consistent with those in Definition 1.

The reduction from Definition~1 to Definition~2 removes the explicit Gaussian regularization term $\|\mathcal{G}\|_{F}^{2}$ and its associated parameter $\lambda$, thereby reducing the number of hyperparameters from three to two. It is important to emphasize that this reduction is not obtained via an iterative approximation or an early stopping strategy. Instead, it is derived from an analytical reformulation of the ADMM updates, in which the Gaussian noise variable can be eliminated without affecting the update rules of the remaining variables. The following theorem formalizes the precise relationship between these two formulations when solved under the ADMM framework.

\newtheorem{theorem}{\bf Theorem}
\begin{theorem}  \label{thm1}
In the ADMM framework, with proper parameter scaling, the denoising frameworks in Definitions 1 and 2 yield identical updates for $\mathcal{X}$, $\mathcal{S}$, and $\mathcal{D}$.
\end{theorem}

This result of Theorem~\ref{thm1} establishes an equivalence at the level of ADMM update rules, rather than a strict equivalence between the two optimization problems. As an immediate consequence, we also present Corollary~\ref{col1}, which follows directly from Theorem~\ref{thm1} and addresses a special case relevant to many practical scenarios.

\newtheorem{corollary}{\bf Corollary}
\begin{corollary}  \label{col1}
Under the same ADMM framework and parameter scaling conditions as Theorem~\ref{thm1}, the updates for ${\mathcal{X}, \mathcal{S}}$ from: $\underset{\mathcal{X}, \mathcal{G}, \mathcal{S}}{\emph{min}}  \mathcal{R}(\mathcal{X}) + \lambda \|\mathcal{G}\|_{F}^{2} + \beta \|\mathcal{S}\|_{1} \hspace{0.3em} \text{s.t.} \hspace{0.3em}\mathcal{Y} = \mathcal{X} + \mathcal{G} + \mathcal{S}, $ coincide with those from: $\underset{\mathcal{X}, \mathcal{S}}{\emph{min}} \hspace{0.2em}  \mathcal{R}(\mathcal{X}) + \beta \|\mathcal{S}\|_{1}
\hspace{0.3em} \text{s.t.} \hspace{0.3em} \mathcal{Y} = \mathcal{X} + \mathcal{S}.$
\end{corollary}

The detailed proof of Theorem~\ref{thm1} is provided in Appendix~A. Although the derivation builds upon standard ADMM properties for $\ell_2$-regularized terms, the theorem yields a concrete and useful insight: the Gaussian noise prior can be absorbed analytically into the ADMM updates, allowing the simplified model to replicate exactly the optimization trajectory of the full model under suitable parameter scaling.

This equivalence yields significant theoretical and practical benefits:
\begin{itemize}
    \item \textbf{Principled model simplification.} Theorem~\ref{thm1} provides a rigorous justification for adopting the reduced formulation by establishing an exact parameter mapping between the two models. It shows that the Gaussian noise component does not need to be treated as a separate variable in the ADMM scheme, offering a mathematically grounded alternative to heuristic simplifications.   
    \item \textbf{Reduced hyperparameter complexity.} Eliminating $\lambda$ reduces the hyperparameter search space from three dimensions ($\lambda,\beta,\gamma$) to two ($\beta,\gamma$). This alleviates a common tuning bottleneck in regularized inverse problems, substantially lowering computational cost and mitigating the risk of overfitting.
\end{itemize}

Consequently, based on the equivalence of the ADMM updates established in Theorem~\ref{thm1}, we adopt the reduced model (Definition~2) for all subsequent implementations and experiments. This ensures that the simplified formulation retains the essential optimization dynamics of the full model while achieving greater tractability through a lower-dimensional parameter space and more efficient computation.

\begin{figure*}[!t]
\centering
\includegraphics[width=5.6in]{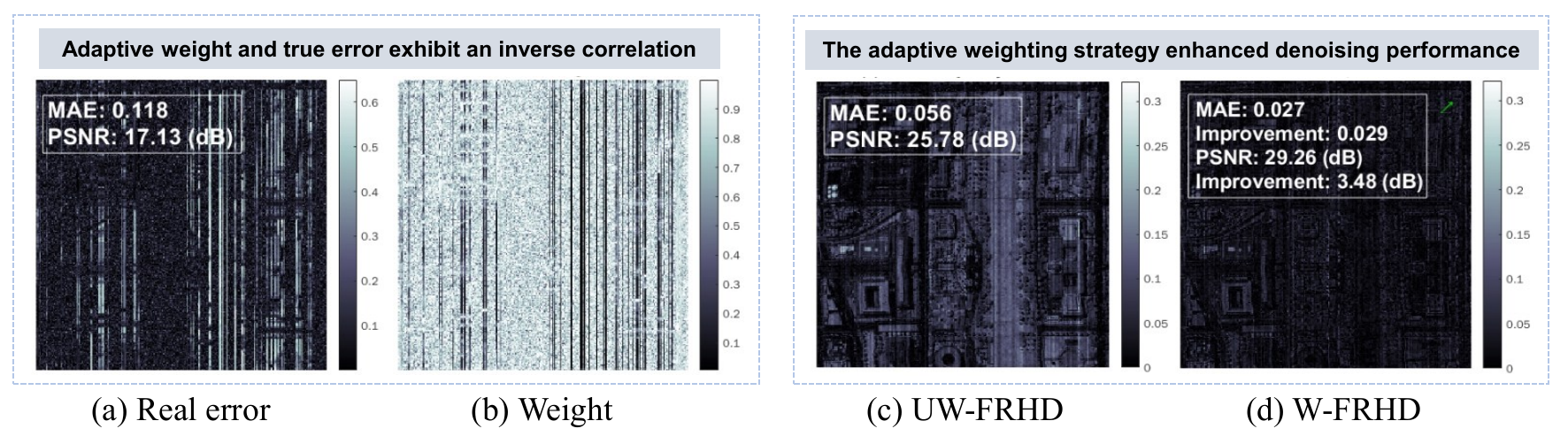}
\caption{Explanation of the adaptive pixel-wise weighting strategy.}
\label{weight_motivation}
\end{figure*}

\subsection{Spatial–Spectral Pixel-wise Adaptive Fidelity.}
In general HSI denoising, the noise modeling within the fidelity term plays a critical role in determining the restoration performance. As the optimization proceeds, noise components are progressively removed across iterations. However, due to the separation of different noise types, solving the noise prior reduction framework~\eqref{NRHSIframework} becomes challenging—particularly in dynamically and stably balancing the interactions among fidelity-related variables. 

To overcome this challenge, we propose a pixel-wise adaptive denoising paradigm that explicitly incorporates adaptive weighting into the fidelity term. The general formulation is defined as follows:

\emph{\textbf{Definition 3} (Pixel-wise denoising paradigm)}: For a observed HSI tensor $\mathcal{Y} \in \mathbb{R}^{M \times N \times B}$, we define the adaptive pixel-wise denoising task as:
\begin{align} \label{PWHSI_denoising}
& \underset{\mathcal{X}, \mathcal{N}, \mathcal{W}}{\text{min}}\,   \overbrace{\mathcal{R}(\mathcal{X})}^{\text{Image Prior}}\hspace{-0.5em} + \hspace{0.2em} \lambda \hspace{-0.4em} \overbrace{\mathcal{P}(\mathcal{N})}^{\text{Noise Prior}} \hspace{-0.5em} + \hspace{0.2em} \alpha \hspace{-0.5em} \overbrace{\mathcal{J}(\mathcal{W})}^{\text{Weight Prior}}\!\!\!, \quad s.t. \hspace{0.8em}  \underbrace{\mathcal{W  \odot  (Y - X)= N}}_{\text{Fidelity Constraints}},
\end{align}
where $\mathcal{R}(\mathcal{X})$ is a regularization term that enforces the spectral smoothness or low-rank structure of the restored image, $\mathcal{P}(\mathcal{N})$ and $\mathcal{J}(\mathcal{W})$ are the noise and weight priors, and $\lambda$ and $\alpha$ are the trade-off parameters that balance the different priors.

Within this framework, we introduce a spatial–spectral pixel-wise adaptive weight tensor $\mathcal{W}$ that dynamically modulates the contribution of each pixel across all bands throughout the optimization process. Unlike fixed or heuristic weighting schemes \cite{chen2022hyperspectral, rui2021learning}, the greatest advantage of our weighting strategy lies in introducing a clever weighting prior term: ${\left\| \mathcal{T}- \mathcal{W} \right\|}$. The key insight behind this prior is that it effectively balances the interplay between the fidelity term and the noise prior without compromising model stability, thereby enhancing adaptability to spatially and spectrally varying noise patterns. As evident from Eq.~\eqref{pixelweighted}, when ($\mathcal{W}=\mathcal{T}$) (i.e., all weights equal one), the model simplifies to the standard unweighted formulation. Conversely, learned weights $\mathcal{W} \neq \mathcal{T}$ activate the prior term $\left\| \mathcal{T\!-\!W} \right\|_{0}$ to regulate the trade-off between data fidelity and noise suppression.

To better elucidate the core mechanism of the adaptive weighting approach, we examine a simplified instantiation of the pixel-level denoising problem:
\begin{equation}  \label{pixelweighted}
\min_{\mathcal{X,W}} \frac{1}{2}\| \mathcal{W} \odot (\mathcal{Y} - \mathcal{X}) \|_F^2 + \lambda \mathcal{R}(\mathcal{X})+ \frac{\alpha}{2}\left\| \mathcal{T\!-\!W} \right\|_{0},
\end{equation} 
where $\mathcal{W} \in [0,1]^{M \times N \times B}$ is the adaptive weighting tensor introduced to modulate fidelity, $\odot$ denotes the Hadamard (element-wise) product, and $\mathcal{T}$ is the unit tensor.

From the formulation in~(\ref{pixelweighted}), it is evident that when the weight tensor equals the unit tensor $\mathcal{T}$, the model reduces to the standard denoising problem. In contrast, when $\mathcal{W} \neq \mathcal{T}$, the interaction between $\mathcal{W}$ and other variables can be characterized by analyzing the solution of its corresponding subproblem:
\begin{equation}  
\min_{\mathcal{W}} \frac{1}{2}\| \mathcal{W} \odot (\mathcal{Y} - \mathcal{X}) \|_F^2 + \frac{\alpha}{2}\left\| \mathcal{T\!-\!W} \right\|_{0}.
\end{equation} 
For the loss function of this subproblem, we observe that when $\mathcal{W}_{ijk} = 1$, the pixel-wise loss is $\frac{1}{2}(\mathcal{Y}_{ijk} - \mathcal{X}_{ijk})^2$, while for $\mathcal{W}_{ijk} = 0$, it becomes a constant $\frac{\alpha}{2}$. Hence, if $\frac{\alpha}{2} < \frac{1}{2}(\mathcal{Y}_{ijk} - \mathcal{X}_{ijk})^2$, the weighted loss is smaller than its unweighted counterpart. This mechanism automatically down-weights pixels with large residuals, which typically correspond to strong noise or outliers.

In practice, the weight update in our full FRHD model (Eq.~\eqref{update:W}) follows a similar principle but operates in a continuous manner. Specifically, the weights are computed from the reconstruction error ($\mathbf{Y} - \mathbf{UV}^\top$) and the noise variables $\mathbf{S}$ and $\mathbf{D}$, following an inverse relationship with the error magnitude. As illustrated in Fig.~\ref{weight_motivation}, regions with higher error—indicative of severe noise contamination—receive lower weights. This spatially and spectrally adaptive suppression prevents heavily corrupted pixels from unduly influencing the model fit, thereby substantially improving robustness under extreme noise.

The above analysis employs the $\ell_0$-norm in Eq.~\eqref{pixelweighted} to induce sparsity on weight deviations, which corresponds to setting $\mathcal{J}(\mathcal{W}) = ||\mathcal{T} - \mathcal{W}||_0$ in the general framework of Definition~3. It is important to note that this choice represents only one particular instantiation of the proposed paradigm. In practice, $\mathcal{J}(\mathcal{W})$ can be flexibly designed to achieve different properties: for instance, using an $\ell_1$-norm promotes convex sparsity and facilitates optimization, while an $\ell_2$-norm encourages smoother spatial-spectral weight variations. This flexibility underscores the generality of the proposed framework, which can be readily integrated into a broad class of HSI restoration models. Further discussion on the advantages of the proposed pixel-wise weighting scheme is provided in the Supplementary Materials.

\subsection{Proposed FRHD model.}
The Noise prior reduction HSI denoising framework accounts for multiple types of noise while significantly reducing the number of regularization parameters. Meanwhile, the proposed spatial–spectral adaptive pixel-weighting strategy balances the fidelity term and prior regularization, thereby improving the accuracy of the iterative solution. Based on this, we construct the FRHD model by integrating the adaptive pixel-wise fidelity term into the noise prior reduction framework and applying efficient and accurate RCTV regularization to the noisy HSI $\mathcal{X}$. Overall, in the formulation of problem~(\ref{PWHSI_denoising}), we define:
\begin{align} 
\left\{
\begin{aligned}
&\mathcal{R}(\mathcal{X}) = \left\| \mathcal{X} \right\|_\text{RCTV}, \lambda \mathcal{P}(\mathcal{N}) = \beta \|\mathcal{S}\|_1 \!+\! \gamma \|\mathcal{D}\|_{2,1};  \\ 
&\mathcal{J}(\mathcal{W}) \!=\! \left\|  \mathcal{\mathcal{T} \!-\!W} \right\|^2_{F}, \hspace{0.1em} \text{and} \hspace{0.4em} \mathcal{W} \odot \mathcal{(Y \!-\! X) \!= \! S \!+\! D},
\end{aligned}
\right.
\end{align}
and expand the tensor into matrix form, where $\mathcal{W} \in [0,1]^{M \times N \times B}$ . 

Thus, the final form of the proposed FRHD model is formulated as follows:
\begin{align} \label{UHSI_LRTV}
& \underset{\mathbf{U},\mathbf{V},\mathbf{W},\mathbf{S},\mathbf{D}}{\text{min}}  \!  \ \sum\limits_{i=1}^{2}{{{\tau }_{i}}}{{\left\| {{\nabla }_{i}}(\mathbf{U}) \right\|}_{1}} \!+\! \alpha {{\left\| \mathbf{\bm{\mathfrak{T}}\!-\!W} \right\|}^2_{F}} \!+\! \beta {{\left\| \mathbf{S} \right\|}_{1}} \!+\! \gamma {{\left\| \mathbf{D} \right\|}_{2,1}} \notag \\
& s.t. \hspace{0.3em} \mathbf{W  \odot  (Y - X) = S + D},  \mathbf{X=UV^{\top}}, \mathbf{V^{\top}V=I},\text{and} \hspace{0.2em} \mathbf{W} \in [0,1]^{MN \times B},
\end{align}
where $\tau_i$, $\alpha$, $\beta$ and $\gamma$ are the penalty parameter, $\mathbf{X=UV^{\top}}$ is the restored clean HSI, $\mathbf{U}$ is the representative coefficients matrix, $\bm{\mathfrak{T}}$ denotes a matrix with all elements equal to 1, which is not necessarily square. Fig. \ref{Flowchart} illustrates the proposed framework and FRHD method for HSI denoising.

\section{Optimization algorithm}   \label{sect4}
\subsection{Alternating Optimization for FRHD Model.}
The solution of Model (\ref{UHSI_LRTV}) is relatively complex, we introduce the auxiliary variables $\mathbf{H}_1,\mathbf{H}_2$ to relax the formulation:
\begin{align}    \label{total_objf}
\!\! \underset{\mathbf{U}\!, \mathbf{V}\!, \mathbf{W}\!,\mathbf{S}\!,\mathbf{D}, \mathbf{H}_i\!}{\mathop{\min }}\ & \hspace{-0.4em} \sum\limits_{i=1}^{2}{{{\tau }_{i}}}{{\left\| {\mathbf{H}_i} \right\|}_{1}} \! + \alpha {{\left\| \mathbf{\bm{\mathfrak{T}}-W} \right\|}^2_{F}} + \beta {{\left\| \mathbf{S} \right\|}_{1}}+ \gamma {{\left\| \mathbf{D} \right\|}_{2,1}} \notag \\
s.t. \hspace{0.7em} &   \mathbf{W\odot (Y-UV^{\top})=S+D}, {{\nabla }_{i}}(\mathbf{U})=\mathbf{H}_i \hspace{0.2em} (i\!=\!1,2), \mathbf{V^{\top}V=I},\text{and} \hspace{0.2em} \mathbf{W} \in [0,1]^{MN \times B}.
\end{align}

The problem (\ref{total_objf}) is a highly challenging optimization problem, as it involves optimizing a multivariable set $\left\{\mathbf{U}^*, \mathbf{V}^*, \mathbf{W}^*, \mathbf{S}^*, \mathbf{D}^*, \{\mathbf{H}_i^*\}_{i=1}^2 \right\}$. To facilitate the solution, we employ the well-known alternating direction method of multipliers (ADMM) and define its augmented Lagrangian function $\bm{\mathcal{L}}$ as:
\begin{align}    \label{Finall_lagrange}
\!\! \bm{\mathcal{L}}&({\left \{\mathbf{U^*,V^*, W^*, S^*,D^*},\{\mathbf{H}_i^*\}^2_{i=1},\{\mathbf{\Lambda}_i^*\}^3_{i=1}\right \}}) \notag\\
&\!\!\!:= \!\!\!\underset{\mathbf{U}\!, \mathbf{V}\!, \mathbf{W}\!,\mathbf{S}\!,\mathbf{D}, \mathbf{H}_i\!}{\mathop{\arg \min }}\ \sum\limits_{i=1}^{2}{{{\tau }_{i}}}{{\left\| {\mathbf{H}_i} \right\|}_{1}} +\beta {{\left\| \mathbf{S} \right\|}_{1}} \!+\! \gamma {{\left\| \mathbf{D} \right\|}_{2,1}} + \frac{\alpha}{2} {{\left\| \bm{\mathfrak{T}}-\mathbf{W} \right\|}^2_{F}} +\! \frac{\rho}{2} \sum\limits_{i=1}^{2}{{{\tau }_{i}}}{{\left\| {{{\nabla }_{i}}(\mathbf{U})\!-\! \mathbf{H}_i} \! + \! \frac{\mathbf{\Lambda}_i}{\rho} \right\|}_{F}^{2}}  \notag \\
 & +  \! \frac{\rho}{2} \! \left\| \mathbf{W\! \odot \!(Y\!-\!UV^{\top}\!)\!-\!S\!-\!D}+\! \frac{\mathbf{\Lambda}_3}{\rho} \right\|_{F}^{2}  \hspace{0.1em} s.t. \hspace{0.1em}  \mathbf{V^{\top}V\!=\!I},\text{and} \hspace{0.2em} \mathbf{W} \in [0,1]^{MN \times B}. 
\end{align}
where $\rho$ is the proximal parameter, and $\mathbf{{\Lambda}_i}$ are the Lagrange multipliers.

In light of this, we now discuss the subproblems associated with each optimization variable.

\vspace{4pt}
$\textbf{1)}$ $\mathbf{H}_i$ $\textbf{Subproblems}$: By fixing the variables $\mathbf{U}$,$\mathbf{V}$,$\mathbf{W}$,$\mathbf{S}$,$\mathbf{D}$, we solve for the auxiliary variable $\mathbf{H}_i$. The corresponding suboptimization problem can be formulated as:
\begin{align}
{{\mathbf{H}_i}}&\!=\!  \underset{\mathbf{H}_i}{\mathop{\arg \min }}\,\sum\limits_{i=1}^{2}{{{\tau }_{i}}}{{\left\| {\mathbf{H}_i} \right\|}_{1}} \!+\! \frac{\rho }{2}\left\| {{\nabla }_{i}}(\mathbf{U}) \!-\! \mathbf{H}_i \!+\! \frac{\mathbf{\Lambda}_i}{\rho}  \right\|_{F}^{2},
\end{align}
where $\mathbf{\Lambda}_i$ is the Lagrange multiplier. This is a standard $\ell_1$-norm regularization problem, which can be efficiently solved using the soft-thresholding algorithm \cite{donoho1995noising}:
\begin{align}   \label{update:H}
\mathbf{H}_i=\mathcal{S}_{{\tau}_{i}/\rho}({\nabla }_{i}(\mathbf{U})+\frac{\mathbf{\Lambda}_i}{\rho}).
\end{align}

$\textbf{2)}$ $\mathbf{U}$ $\textbf{Subproblems}$: Similarly, by fixing all other variables, the suboptimization problem for $\mathbf{U}$ can be stated as:
\begin{align}  \label{uini}
\mathbf{U}_i =  \hspace{0.2em} \underset{\mathbf{U}_i}{\mathop{\arg \min }}\ \frac{\rho}{2} \left\| \mathbf{W  \odot  (Y - UV^{\top}\!) - S - D} +  \frac{\mathbf{\Lambda}_3}{\rho} \right\|_{F}^{2}+ \frac{\rho}{2} \sum\limits_{i=1}^{2}{{{\tau }_{i}}}{{\left\| {{{\nabla }_{i}}(\mathbf{U})-\mathbf{H}_i} \!+\! \frac{\mathbf{\Lambda}_i}{\rho} \right\|}_{F}^{2}}.
\end{align}

Taking the derivative of equation (\ref{uini}) with respect to $\mathbf{U}$, we get:
\begin{align} \label{usubp}
\left( \rho \mathbf{I} + \rho   \sum_{i=1}^{2} {\nabla }_{i}^{\top}{\nabla }_{i}  \right) &  (\mathbf{U})  =  \rho (\mathbf{Y - (S + D)/W} + {\frac{\mathbf{\Lambda}_3}{\rho}})\mathbf{V}  + \sum_{i=1}^{2} {\nabla }_{i}^{\top}(\rho \mathbf{H}_i - \mathbf{\Lambda}_i).
\end{align}

By applying the Fourier transform to both sides of equation (\ref{usubp}) and utilizing the convolution theorem, the closed-form solution for $\mathbf{U}$ is easily deduced as \cite{NIPS2009}:
\begin{align}  \label{update:U}
\begin{cases}
\mathbf{R} = \sum_{i=1}^{2} \mathcal{F}(\mathbf{F}_i)^\ast \odot \mathcal{F}(\text{fold}(\rho \mathbf{H}_i)-\mathbf{\Lambda}_i), \\[3pt]
\mathcal{U} = \mathcal{F}^{-1}\left(\frac{\mathcal{F}(\text{fold}((\mathbf{Y-(S+D)/W+\mathbf{\Lambda}_3/ \rho)}\mathbf{V}))+ \mathbf{R}} {\rho \bm{\mathfrak{T}} + \rho(|\mathcal{F}(\mathbf{F}_1)|^2+ |\mathcal{F}(\mathbf{F}_2)|^2 )}\right), \\[3pt]
\mathbf{U} = \text{unfold}_3(\mathcal{U}),
\end{cases}
\end{align}
where $\mathcal{F}(\cdot)$ is the Fourier transform, $\mathbf{F}_i$ is the differential filter, and $|\cdot|^2$ is the element-wise squaring operation.

\vspace{4pt}
$\textbf{3)}$ $\mathbf{V}$ $\textbf{Subproblems}$: Fixing the other variables, the optimization of $\mathbf{V}$ can be expressed as follows:
\begin{align}  \label{svd}
{\mathbf{V}}\!=\! \underset{\mathbf{V}}{\mathop{\arg \min }}\,\frac{\rho }{2}\left\| \mathbf{W \! \odot \! (Y\!-\!UV^{\top}\!)\!-\!S\!-\!D}\!+\!\!{\frac{\mathbf{\Lambda}_3}{\rho}} \right\|_{F}^{2}   s.t. \hspace{0.1em} {{\mathbf{V}}^{\top }}\!\mathbf{V} \!=\! \mathbf{I}.
\end{align}

The optimization problem for $\mathbf{V}$ is equivalent to the following formula:
\begin{align}  \label{equivalent}
{\mathbf{V}}= \underset{{{\mathbf{V}}^{\top }}\mathbf{V}=\mathbf{I}}{\mathop{\arg \max }}\, \!<\! \mathbf{Y\!-\!(S\!+\!D)/W} \!+\! {\frac{\mathbf{\Lambda}_3}{\rho}},\mathbf{UV^{\top}}\!\!>\!\!.
\end{align}

Then, we can perform SVD to efficiently obtain the closed-form solution:
\begin{align} \label{update:V}
\begin{cases}
[\mathbf{B}, \mathbf{\Sigma}, \mathbf{C}] = \text{SVD}\left((\mathbf{Y\!-\!(S\!+\!D)/W}+\! {\frac{\mathbf{\Lambda}_3}{\rho}})^\top \mathbf{U}\right), \\
\mathbf{V} = \mathbf{B} \mathbf{C}^\top.
\end{cases}
\end{align}

$\textbf{4)}$ $\mathbf{S}$ $\textbf{Subproblems}$: The subproblem for $\mathbf{S}$ can be described as follows:
\begin{align}
{\mathbf{S}}\!=\! \underset{\mathbf{S}}{\mathop{\arg \min }}\,\frac{\rho }{2}\left\| \mathbf{W \! \odot \! (Y\!-\!UV^{\top}\!)\!-\!S\!-\!D}\!+\!\!{\frac{\mathbf{\Lambda}_3}{\rho}} \right\|_{F}^{2} \!+\! \beta {{\left\| \mathbf{S} \right\|}_{1}},
\end{align}
this is also a typical $\ell_1$-norm problem, which can be solved using the soft-thresholding algorithm:
\begin{align}   \label{update:S}
\mathbf{S}=\mathcal{S}_{\beta/\rho}(\mathbf{W\! \odot \!(Y\!\!-\!\!UV^{\top}\!)\!-\!D}+\!{\frac{\mathbf{\Lambda}_3}{\rho}}).
\end{align}

$\textbf{5)}$ $\mathbf{D}$ $\textbf{Subproblems}$: By fixing all other variables, the subproblem for $\mathcal{D}$ can be expressed as:
\begin{align}
{\mathbf{D}}\!=\! \underset{\mathbf{D}}{\mathop{\arg \min }}\, \frac{\rho }{2}\left\| \mathbf{W \! \odot \! (Y\!-\!UV^{\top}\!)\!-\!S\!-\!D}\!+\!\!{\frac{\mathbf{\Lambda}_3}{\rho}} \right\|_{F}^{2} \!+\! \gamma {{\left\| \mathbf{D} \right\|}_{2,1}}.
\end{align}

According to the Lemma 4.1 of \cite{liu2013robust}, the solution to the $\ell_{2,1}$-norm problem is:
\begin{align}   \label{update:D}
\mathbf{D}\left( :,\mathbf{j} \right)=\left\{ 
\begin{matrix}
\frac{\left\| \mathbf{\Omega_j} \right\|_{F}^{2}-\gamma }{\left\| \mathbf{\Omega_j} \right\|_{F}^{2}}, \hspace{0.8em}  \text{if} \hspace{0.4em}  \gamma \le \left\| \mathbf{\Omega_j} \right\|_{F}^{2};   \\
\hspace{0.1em} 0,    \hfill    \text{otherwise}.  \\
\end{matrix} \right.
\end{align}
where $\mathbf{\Omega_j}=\left\{\mathbf{W}\! \odot \!(\mathbf{Y}- \mathbf{U}{\mathbf{V}}^{\top})\right\}\mathbf{(:,j)}-\mathbf{S(:,j)}+{\frac{\mathbf{\Lambda}_3(:,\mathbf{j})}{\rho}}$.\\

$\textbf{6)}$ $\mathbf{W}$ $\textbf{Subproblems}$: Extracting all terms related to $\mathcal{W}$, we obtain the following subproblem:
\begin{align}    
{\mathbf{W}}\!=\! \underset{\mathbf{W}}{\mathop{\arg \min }}\, \frac{\rho }{2}\left\| \mathbf{W \! \odot \! (Y\!-\!UV^{\top}\!)\!-\!S\!-\!D}\!+\!\!{\frac{\mathbf{\Lambda}_3}{\rho}} \right\|_{F}^{2} \!\!\!+\! \frac{\alpha}{2} {{\left\| \bm{\mathfrak{T}}- \mathbf{W} \right\|}^2_{F}},s.t. \hspace{0.2em} \mathbf{W} \in [0,1]^{MN \times B}.
\end{align}

Since $\odot$ is the Hadamard product, this leads to a straightforward least squares problem, with the solution given by:
\begin{align}   \label{update:W}
\mathbf{W}= \mathcal{P}_{[0,1]}\left( \frac{\alpha \bm{\mathfrak{T}} + \left[ \rho \hspace{0.2em} (\mathbf{S+D-{\frac{\mathbf{\Lambda}_3}{\rho}}})\odot(\mathbf{Y-UV}^{\top}) \right]}{\alpha \bm{\mathfrak{T}}+ \rho\left(\mathbf{Y-UV}^{\top}\right)\odot(\mathbf{Y-UV}^{\top})} \right).
\end{align}
Here, the division is performed element-wise, and $\mathcal{P}_{[0,1]}$ denotes entrywise projection onto $[0,1]$. In practice, we observe that the updated weight $\mathbf{W}$ inherently satisfies the $[0,1]$ interval constraint without requiring explicit projection. This empirical boundedness can be attributed to the experimental setup—both input images and added noise were normalized to $[0,1]$—and the structure of the closed‑form solution, where the denominator remains strictly positive and empirically dominates the numerator. Throughout all experiments, tensor $\mathcal{W}$ consistently remained within this physically interpretable range, confirming its validity as a pixel‑wise fidelity weight.

$\textbf{7) Lagrange Multiplier Update}$: Finally, the Lagrange multipliers are updated as follows:
\begin{align}  \label{update:Multipliers}
\begin{cases}
\mathbf{\Lambda}_i = \mathbf{\Lambda}_i + \rho ({{\nabla }_{i}}(\mathbf{U})-\mathbf{H}_i), i=1,2, \\
\mathbf{\Lambda}_3 = \mathbf{\Lambda}_3 + \rho (\mathbf{W \!\odot\! (Y\!-\!UV^{\top})\!-\!S\!-\!D}),\\
\rho = \eta \rho.
\end{cases}
\end{align}

From the closed-form solution of $\mathbf{W}$ in (\ref{update:W}), it is evident that the weight is inversely proportional to the residual $(\mathbf{Y} - \mathbf{U}\mathbf{V}^\top)$ between the observed data and the reconstructed image. This design adaptively assigns smaller weights to pixels with larger residuals, thereby reducing their contribution to the fidelity term and effectively suppressing extreme noise perturbations.

After completing all the iterations, the clean HSI $\mathbf{X = UV^{\top}}$ can be obtained. The entire optimization process is summarized in Algorithm \ref{alg:FRHD}.

\begin{figure}[!t] 
\centering
\subfloat[]{\includegraphics[width=0.9in]{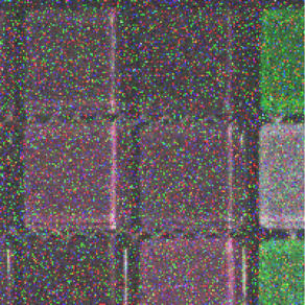}%
\label{case1}}
\hspace{1em}
\subfloat[]{\includegraphics[width=0.9in]{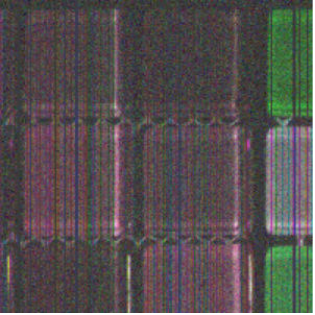}%
\label{case2}}
\hspace{1em}
\subfloat[]{\includegraphics[width=0.9in]{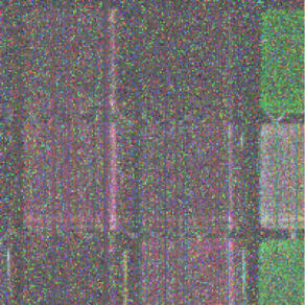}%
\label{case3}}
\hspace{1em}
\subfloat[]{\includegraphics[width=0.9in]{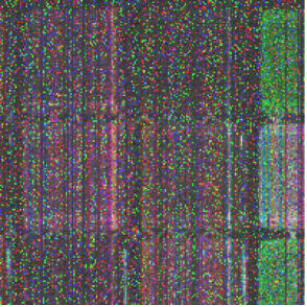}%
\label{case4}}
\hspace{1em}
\subfloat[]{\includegraphics[width=0.9in]{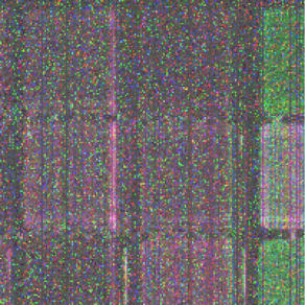}%
\label{case5}}
\caption{Five different levels of mixed noise pollution on the CAVE datasets, the R, G, and B channels correspond to spectral bands 14, 23, and 16, respectively. (a) Case 1; (b) Case 2; (c) Case 3; (d) Case 4; (e) Case 5.}
\label{noise_case}
\end{figure}

\subsection{Computational complexity analysis.}
In this subsection, we briefly analyze the computational complexity of the proposed FRHD algorithm. As outlined in Algorithm \ref{alg:FRHD}, the primary computational cost per iteration arises from steps 2-8. Specifically, steps 2 and 5 involve soft-thresholding operators, with an update complexity of $\mathcal{O}(MNB)$. Step 3, which updates $\mathbf{U}$ via a simple FFT, has a complexity of $\mathcal{O}(MNR\log(MN))$. Additionally, step 4 entails an SVD of a reduced matrix, which incurs a cost of $\mathcal{O}(BR^2)$, where $R$ denotes the rank of the SVD. Finally, steps 6, 7, and 8 involve straightforward matrix operations, contributing an overall complexity of approximately $C \cdot \mathcal{O}(MNB)$. Accordingly, the total computational complexity of the proposed algorithm is: $\mathcal{O}(MNB + BR^2 + MNR \log(MN))$. In summary, the proposed refined noise prior and adaptive pixel-wise weighting strategy incur no additional computational complexity. Moreover, the fidelity constraint avoids large-scale SVD computations and applies TV regularization only to $\mathbf{U}$, thus reducing computational overhead.

{\centering
\begin{minipage}{\linewidth}
\renewcommand{\algorithmiccomment}[1]{$\triangleright$ #1} 
\begin{algorithm}[H]
\caption{ADMM solver for the FRHD model.}
\label{alg:FRHD}
\begin{algorithmic}[1]
\item[] \hspace{-0.5cm} \textbf{Input:} Observed image $\mathcal{Y} \! \in \! \mathbb{R}^{M  \times  N  \times  B}$,  SVD rank $r$, regularization parameters $\tau$\footnotemark, $\beta,\gamma$, and other fixed model parameters $\alpha=1$, $\rho = \frac{1}{\lambda_{\text{max}}(\mathbf{Y})}$, $\eta = 2$, $\epsilon = 10^{-5}$, $\text{maxIter} = 100$.
\item[] \hspace{-0.5cm} \textbf{Initialization:} Unfold the image $\mathcal{Y}$ and weight  tensor $\mathcal{W}=\mathfrak{T}$ along mode-3 to obtain matrix form, and perform SVD on the unfolded $\mathbf{Y}   \in   \mathbb{R}^{MN  \times   B}$ to obtain the initial $\mathbf{U} \! \in \! \mathbb{R}^{MN   \times  R}$ and $\mathbf{V}  \in  \mathbb{R}^{B   \times   R}$.
\WHILE{convergence criterion is not met}
\STATE Update $\mathbf{H}_i$.  \hspace{8.18em}   \COMMENT{By (\ref{update:H})}.
\STATE Update $\mathbf{U}$.    \hspace{8.55em}   \COMMENT{By (\ref{update:U})}.
\STATE Update $\mathbf{V}$.    \hspace{8.55em}   \COMMENT{By (\ref{update:V})}.
\STATE Update $\mathbf{S}$.    \hspace{8.79em}   \COMMENT{By (\ref{update:S})}.
\STATE Update $\mathbf{D}$.   \hspace{8.58em}    \COMMENT{By (\ref{update:D})}.
\STATE Update $\mathbf{W}$.    \hspace{8.27em}   \COMMENT{By (\ref{update:W})}.
\STATE Lagrange Multipliers update.  \hspace{1.12em}   \COMMENT{Via (\ref{update:Multipliers})}.
\STATE Determine the convergence conditions:\\
\hspace{0.3cm}  $\left\| {{{\nabla }_{i}}(\mathbf{U})\!-\! \mathbf{H}_i} \right\|_{F}^{2} \! /\! \left\| \mathbf{Y} \right\|_{F}^{2} \leq \epsilon, i=1,2,$ \\ 
\hspace{0.3cm}  and $\left\| \mathbf{W\! \odot \!(Y\!-\!UV^{\top}\!)\!-\!S\!-\!D} \right\|_{F}^{2}\! /\! \left\| \mathbf{Y} \right\|_{F}^{2}  \leq \epsilon.$
\STATE $iter \leftarrow iter + 1.$
\ENDWHILE
\item[] \hspace{-0.5cm} \textbf{Output:} Recovered image $\mathcal{X}= \text{fold3}(\mathbf{U}\mathbf{V^{\top}}) \in \mathbb{R}^{M \times N \times B}$.
\end{algorithmic}
\end{algorithm}
\vspace{-1.8em}
\footnotetext{ \textbf{\emph{Note}}: The parameter $\tau = [\tau_1, \tau_2]$ is essentially a two-dimensional vector. For simplicity and ease of tuning, we set $\tau_1 = \tau_2$, reducing $\tau$ to a scalar.}
\end{minipage}}

\textbf{\emph{Remark:}} Step~9 in Algorithm~\ref{alg:FRHD} defines the stopping criterion for the FRHD model. In classical regularization theory, the stopping criterion is typically related to the noise level \cite{boyd2011distributed}. Here, Step~9 is designed to ensure practical convergence of the ADMM iterations. We emphasize that this criterion is purely a computational consideration: it does not affect the theoretical equivalence between Definitions~1 and~2 established in Theorem~\ref{thm1}, nor does it serve as a justification for the model reduction. In practice, this stopping rule provides a reliable and efficient termination condition, especially given the known slowdown of ADMM beyond a certain accuracy threshold.

\section{Experiment and result analysis}  \label{sect5}
In this section, we systematically evaluate the performance of the proposed FRHD method through extensive experiments on three widely recognized HSI simulation datasets and a prominent real-world dataset. To comprehensively highlight the advantages of our method, we compare it against eight representative mixed noise HSI restoration methods proposed in the past five years, including SDeCNN~\cite{Maffei2020}, WNLRATV\cite{chen2022hyperspectral}, RCTV \cite{peng2022fast}, FastHyMix \cite{zhuang2023fasthymix}, TPTV \cite{chen2023TPTV}, and CTV-SPCP \cite{peng2024stable}, FBGND~\cite{Liang2024}, and FallHyDe~\cite{ChenYong2024}. Among them, SDeCNN, FastHyMix, and FBGND are deep learning-based methods that leverage either purely data-driven architectures or hybrid model-driven priors. The parameters of all competing methods were fine-tuned based on the original authors' code implementations or the parameter configurations recommended in their respective publications. All experiments were conducted on a high-performance computing platform, equipped with an Intel Core i9-11900 CPU and an NVIDIA RTX 3090 GPU, supported by 24 GB of RAM.

\begin{figure*}[!t]
\centering
\includegraphics[width=6in]{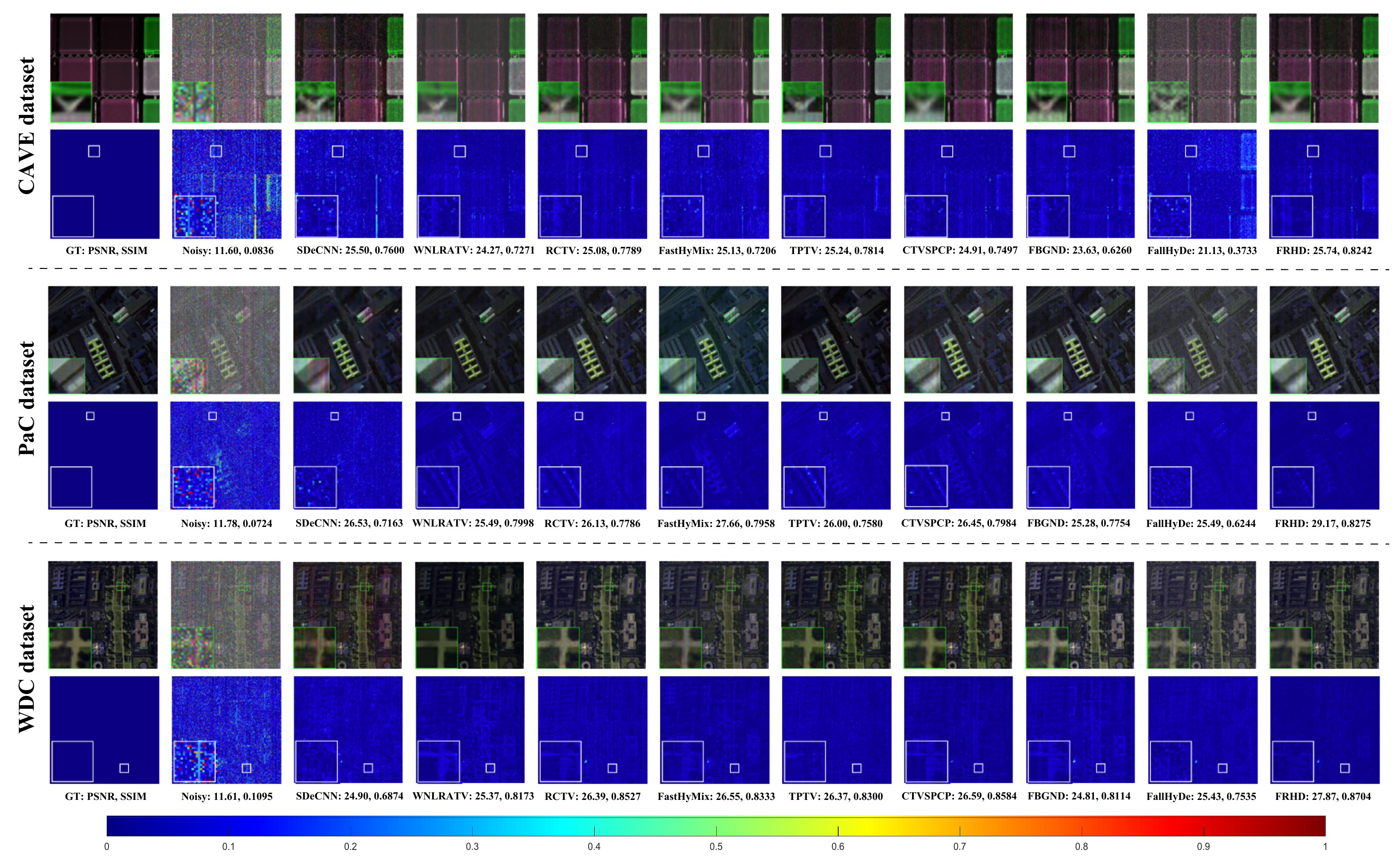}
\caption{Denoising results of all competing methods on the simulated datasets under Case~5. The pseudo-color images are visualized using spectral bands 14/23/16 (CAVE), 6/20/50 (PaC), and 152/106/20 (WDC) for the R-G-B channels, respectively.}
\label{Simulated_results}
\end{figure*}

\begin{table*}[!t]
\fontfamily{ptm}\selectfont 
\caption{Comparison of Denoising Methods on Simulated datasets. Best and second-best results are highlighted in \textbf{bold} and \underline{underline}, respectively. \label{tab:simulated}}
\centering
\renewcommand{\arraystretch}{1}
\Large
\resizebox{0.85\columnwidth}{!}{
\begin{tabular}{c||cccccccccccc}
\Xhline{1.15pt}  \noalign{\vskip 0.3ex}
\rowcolor[HTML]{DCDCDC}\cellcolor[HTML]{FFFFFF} Datasets & Case & Indexes & Noisy  & \parbox{2cm}{\centering SDeCNN\\ ~\cite{Maffei2020}} & \parbox{2.2cm}{\centering WNLRATV\\ \cite{chen2022hyperspectral}} & \parbox{2cm}{\centering RCTV\\ \cite{peng2022fast}} & \parbox{2.1cm}{\centering FastHyMix\\ \cite{zhuang2023fasthymix}} & \parbox{2cm}{\centering TPTV\\ \cite{chen2023TPTV}}  & \parbox{2.4cm}{\centering CTV-SPCP\\ \cite{peng2024stable}} &  \parbox{2cm}{\centering FBGND\\ \cite{Liang2024}} &  \parbox{2cm}{\centering FallHyDe\\ \cite{ChenYong2024}} & \parbox{2cm}{\centering FRHD \\ Ours} \\
\cline{1-13} 
\multirow{20}{*}{\rotatebox[origin=c]{90}{{\boldmath\textbf{CAVE dataset $(200\times 200\times 31)$}}}} 
& \multirow{4}{*}{\parbox{2cm}{Case 1 \\ \small G:0.1+S:0.15}}
& PSNR $\uparrow$   & 12.1544  & 24.6370 & 31.1530 & \underline{31.2021} & 28.9673  & 31.0676  & 30.1415 & 28.2611	 & 21.0550 & \textbf{31.8896} \\
&& SSIM $\uparrow$   & 0.1125   & 0.7956  & 0.8271  & 0.9139  & 0.9011  & \underline{0.9194}  & 0.8493 & 0.8339	& 0.3994  & \textbf{0.9195}  \\
&& SAM $\downarrow$  & 0.8056   & 0.1967 & 0.1887  & \underline{0.1381}  & 0.1933  & 0.1910 & 0.1598  & 0.1949 & 0.3357 & \textbf{0.1202}  \\
& &  \cellcolor[HTML]{E8F0F7}Time   &  \cellcolor[HTML]{E8F0F7} --  &  \cellcolor[HTML]{E8F0F7}5.4347  &  \cellcolor[HTML]{E8F0F7}69.3791 &  \cellcolor[HTML]{E8F0F7}0.9321  &  \cellcolor[HTML]{E8F0F7}0.2315  &  \cellcolor[HTML]{E8F0F7}11.9730  &  \cellcolor[HTML]{E8F0F7}6.4013	 &  \cellcolor[HTML]{E8F0F7}323.4148 &  \cellcolor[HTML]{E8F0F7}0.0401 &  \cellcolor[HTML]{E8F0F7}0.6290  \\			 
\cline{2-13} 
&\multirow{4}{*}{\parbox{2cm}{Case 2 \\ \small G:0.1+D:0.2}}
& PSNR $\uparrow$   & 17.2448   & 25.0311 & 26.0098 & \underline{26.3044} & 24.4271  & \textbf{26.5381}  & 25.7595 & 24.0865 &	22.8881 & 25.0156 \\
&& SSIM $\uparrow$   & 0.2087    & \textbf{0.8381}  & 0.8125  & 0.8219  & 0.7570   & 0.8272   & 0.7835 & 0.6979 &	0.6454 & 0.\underline{8277}  \\
&& SAM $\downarrow$  & 0.7369  & \underline{0.1670} & 0.2049  & 0.1785  & 0.1883  & 0.2144 & 0.1984 & 0.2426  & 0.3101   & \textbf{0.1624}  \\
& &  \cellcolor[HTML]{E8F0F7}Time   &  \cellcolor[HTML]{E8F0F7}--   &  \cellcolor[HTML]{E8F0F7}5.4180  &  \cellcolor[HTML]{E8F0F7}75.0663 &  \cellcolor[HTML]{E8F0F7}0.9221  &  \cellcolor[HTML]{E8F0F7}0.1751  &  \cellcolor[HTML]{E8F0F7}12.3568  &  \cellcolor[HTML]{E8F0F7}6.1942  &  \cellcolor[HTML]{E8F0F7}339.1856 &  \cellcolor[HTML]{E8F0F7}0.0387 &  \cellcolor[HTML]{E8F0F7}1.1010 \\
\cline{2-13} 
&\multirow{4}{*}{\parbox{2.35cm}{Case 3 \\ \small G:0.1+S:0.2+D:0.2}}
& PSNR $\uparrow$   & 11.5339  & \textbf{27.3106} & 25.7458 & 25.9984 & \underline{27.1866}  & 25.9774  & 25.5053 & 24.3361 & 	21.5859 & 26.6675 \\
&& SSIM $\uparrow$   & 0.0845   & 0.8128  & 0.7787  & 0.8062  & \underline{0.8240}  & 0.8205  & 0.7506 & 0.6108 & 	0.3849  & \textbf{0.8315}  \\
&& SAM $\downarrow$  & 0.9345   & \underline{0.1831}  & 0.2317  & 0.1919 & 0.1952  & 0.2398   & 0.2157 & 0.2470  & 0.4016 & \textbf{0.1653}  \\
&& Time  \cellcolor[HTML]{E8F0F7} & \cellcolor[HTML]{E8F0F7}--  &  \cellcolor[HTML]{E8F0F7}5.8830 &\cellcolor[HTML]{E8F0F7}	69.6415 &\cellcolor[HTML]{E8F0F7}	0.8898 &\cellcolor[HTML]{E8F0F7}	0.2276 &\cellcolor[HTML]{E8F0F7}	11.8430 &\cellcolor[HTML]{E8F0F7}	5.6994 &\cellcolor[HTML]{E8F0F7}	321.9835 &\cellcolor[HTML]{E8F0F7}	0.0335 &\cellcolor[HTML]{E8F0F7}	0.7944  \\
\cline{2-13} 
&\multirow{4}{*}{\parbox{2.35cm}{Case 4 \\ \small G:0.2+S:0.1+D:0.1}}
& PSNR $\uparrow$   & 11.3993  & 23.6142 & 21.7691 & 23.9657 & 23.1746 & \underline{24.1669} & 23.6516 & 22.5200 &	20.4308  & \textbf{24.4118} \\
&& SSIM $\uparrow$   & 0.0762    & 0.7064  & 0.6294  & 0.7674  & 0.7211   & \underline{0.7707}  & 0.7445 & 0.6378 &	0.3475  & \textbf{0.8200}  \\
&& SAM $\downarrow$  & 0.9319   & 0.2308  & 0.3316  & \underline{0.1994}  & 0.2439  & 0.2546  & 0.2032 & 0.2420 & 0.4038 & \textbf{0.1654}  \\
&& \cellcolor[HTML]{E8F0F7} Time   &\cellcolor[HTML]{E8F0F7} --  &\cellcolor[HTML]{E8F0F7} 4.9360	 &\cellcolor[HTML]{E8F0F7} 67.3621  &\cellcolor[HTML]{E8F0F7}	0.8791  &	\cellcolor[HTML]{E8F0F7} 0.1800  &\cellcolor[HTML]{E8F0F7}	11.3497  &\cellcolor[HTML]{E8F0F7}	5.5136  &\cellcolor[HTML]{E8F0F7}	359.6305  &\cellcolor[HTML]{E8F0F7}	0.0346  &\cellcolor[HTML]{E8F0F7}	0.7972  \\
\cline{2-13} 
&\multirow{4}{*}{\parbox{2.35cm}{Case 5 \\ \footnotesize G:0.15+S:0.15+D:0.15}} 
& PSNR $\uparrow$   & 11.6046  & \underline{25.4982} & 24.2657 & 25.0842 & 25.1295  & 25.2404   & 24.9095 & 23.6310	 & 21.1309 & \textbf{25.7411} \\
&& SSIM $\uparrow$   & 0.0836    & 0.7600  & 0.7281  & \underline{0.7789}  & 0.7206   & 0.7815   & 0.7497 & 0.6260 & 0.3733 & \textbf{0.8242}  \\
&& SAM $\downarrow$  & 0.9221 & 0.2089 & 0.2335  & \underline{0.1997}  & 0.2298   & 0.2568  & 0.2074 & 0.2521 & 0.3553 & \textbf{0.1694}  \\
&&  \cellcolor[HTML]{E8F0F7}Time   &  \cellcolor[HTML]{E8F0F7}--   & \cellcolor[HTML]{E8F0F7} 5.5871  & \cellcolor[HTML]{E8F0F7}	69.1769 &	 \cellcolor[HTML]{E8F0F7}0.9060 & \cellcolor[HTML]{E8F0F7} 0.2141 & \cellcolor[HTML]{E8F0F7}	11.9354	&  \cellcolor[HTML]{E8F0F7}6.4682 & \cellcolor[HTML]{E8F0F7}	320.6330 & \cellcolor[HTML]{E8F0F7}	0.0407 & \cellcolor[HTML]{E8F0F7}	0.8321 \\
\hline\hline

\multirow{20}{*}{\rotatebox[origin=c]{90}{{\boldmath\textbf{PaC dataset $(300\times 300\times 103)$}}}}
& \multirow{4}{*}{\parbox{2cm}{Case 1 \\ \small G:0.1+S:0.15}}
& PSNR $\uparrow$ & 12.2298 & 24.8823 & 34.0084 & \underline{34.8450} & 31.7299 & 34.1327 & 33.8679  & \textbf{35.5385}  & 24.0416 & 34.5108 \\
&& SSIM $\uparrow$ & 0.0994  & 0.7512 & \underline{0.9254} & 0.9192 & 0.9062 & 0.9060  & 0.9085 & \textbf{0.9355} & 0.5943 & 0.9166 \\
&& SAM $\downarrow$& 0.8268 & 0.1624 & 0.0864 & 0.0863 & 0.1375 & 0.0929  & 0.0871  & \textbf{0.0717} & 0.1981 & \underline{0.0861} \\
&&\cellcolor[HTML]{E8F0F7} Time &\cellcolor[HTML]{E8F0F7} -- &\cellcolor[HTML]{E8F0F7} 37.5005 &\cellcolor[HTML]{E8F0F7} 276.7751 &\cellcolor[HTML]{E8F0F7} 5.7616 &\cellcolor[HTML]{E8F0F7} 0.4564 &\cellcolor[HTML]{E8F0F7} 95.9312 &\cellcolor[HTML]{E8F0F7} 71.9459 &\cellcolor[HTML]{E8F0F7} 2133.67 &\cellcolor[HTML]{E8F0F7} 0.6492 &\cellcolor[HTML]{E8F0F7} 4.3196 \\
\cline{2-13}  
&\multirow{4}{*}{\parbox{2cm}{Case 2 \\ \small G:0.1+D:0.2}}
& PSNR $\uparrow$ & 18.0195 & 26.6129 & 27.5526 & \underline{28.1750} & 27.1438 & 27.8549 & 27.4137  & 26.9077 & 26.9197 & \textbf{29.0473} \\
&& SSIM $\uparrow$ & 0.2246 & 0.8037 & 0.8593 & 0.8502 & \underline{0.8630} & 0.8315  & 0.8386 & 0.8347 & 0.8217  & \textbf{0.8695} \\
&& SAM $\downarrow$& 0.6937  & 0.1342 & \underline{0.1030} & 0.1281 & \textbf{0.0975} & 0.1356   & 0.1344  & 0.1303  & 0.1278 & 0.1159 \\
& &\cellcolor[HTML]{E8F0F7} Time &\cellcolor[HTML]{E8F0F7} -- &\cellcolor[HTML]{E8F0F7} 36.5359 &\cellcolor[HTML]{E8F0F7} 275.7437 &\cellcolor[HTML]{E8F0F7} 5.9733 &\cellcolor[HTML]{E8F0F7} 0.4650 &\cellcolor[HTML]{E8F0F7} 97.2199 &\cellcolor[HTML]{E8F0F7} 72.7873 &\cellcolor[HTML]{E8F0F7} 2332.61 &\cellcolor[HTML]{E8F0F7} 0.5427 &\cellcolor[HTML]{E8F0F7} 6.5807 \\
\cline{2-13}  
&\multirow{4}{*}{\parbox{2.35cm}{Case 3 \\ \small G:0.1+S:0.2+D:0.2}}
& PSNR $\uparrow$ & 11.6527  & 28.0705 & 27.1300 & 27.2854 & \underline{29.1074}  & 27.1198  & 27.8564 & 26.6919 & 25.8819 & \textbf{29.9047} \\
&& SSIM $\uparrow$ & 0.0711 & 0.7687 & 0.8123 & 0.7886 & \textbf{0.8325}  & 0.7775  & 0.8086 & 0.7995 & 0.5825  & \underline{0.8286} \\
&& SAM $\downarrow$& 0.9442 & 0.1546  & \underline{0.1152} & 0.1459 & 0.1352  & 0.1509  & 0.1278 & 0.1283 & 0.2143  & \textbf{0.1143} \\
&& \cellcolor[HTML]{E8F0F7}Time &\cellcolor[HTML]{E8F0F7} --  &\cellcolor[HTML]{E8F0F7} 36.4831 &\cellcolor[HTML]{E8F0F7} 284.4584 &\cellcolor[HTML]{E8F0F7} 6.6499 &\cellcolor[HTML]{E8F0F7} 0.5155 &\cellcolor[HTML]{E8F0F7} 110.0189 &\cellcolor[HTML]{E8F0F7} 77.8764 &\cellcolor[HTML]{E8F0F7} 2341.94 &\cellcolor[HTML]{E8F0F7} 0.5676 &\cellcolor[HTML]{E8F0F7} 6.4989   \\
\cline{2-13}   
&\multirow{4}{*}{\parbox{2.35cm}{Case 4 \\ \small G:0.2+S:0.1+D:0.1}}
& PSNR $\uparrow$ & 11.6556  & 25.2393 & 24.6234 &  \underline{26.0115} & 24.4807  & 25.8435  & 25.6160 & 24.8161 & 24.3879 & \textbf{28.6153} \\
&& SSIM $\uparrow$ & 0.0695 & 0.6649 & 0.7818 & \underline{0.7868} & 0.7530  & 0.7629  & 0.7827 & 0.7670  & 0.5362  & \textbf{0.8289} \\
&& SAM $\downarrow$& 0.9462  & 0.2052  & \underline{0.1242} & 0.1541 & 0.2198  & 0.1663  & 0.1598 & 0.1492  & 0.2424  & \textbf{0.1181} \\
&&\cellcolor[HTML]{E8F0F7} Time &\cellcolor[HTML]{E8F0F7} --  &\cellcolor[HTML]{E8F0F7} 36.6441 &\cellcolor[HTML]{E8F0F7} 281.6717 &\cellcolor[HTML]{E8F0F7} 5.6740 &\cellcolor[HTML]{E8F0F7} 0.4372 &\cellcolor[HTML]{E8F0F7} 106.2056 &\cellcolor[HTML]{E8F0F7} 71.4083 &\cellcolor[HTML]{E8F0F7} 2366.41 &\cellcolor[HTML]{E8F0F7} 0.5264 &\cellcolor[HTML]{E8F0F7} 5.5996  \\
\cline{2-13}  
&\multirow{4}{*}{\parbox{2.35cm}{Case 5 \\ \footnotesize G:0.15+S:0.15+D:0.15}} 
& PSNR $\uparrow$ & 11.7823 & 26.5291 & 25.4944 & 26.1311 & \underline{27.6642}  & 26.0009  & 26.4476  & 25.2826 & 25.4886 & \textbf{29.1734} \\
&& SSIM $\uparrow$ & 0.0724 & 0.7163  & \underline{0.7998} & 0.7786 & 0.7958 & 0.7580  & 0.7984  & 0.7754  & 	0.6244  & \textbf{0.8257} \\
&& SAM $\downarrow$ & 0.9357 & 0.1792 & \underline{0.1178} & 0.1555 & 0.1648 & 0.1651 & 0.1367  & 0.1425	 & 0.2243 & \textbf{0.1172} \\
&&\cellcolor[HTML]{E8F0F7} Time &\cellcolor[HTML]{E8F0F7} -- &\cellcolor[HTML]{E8F0F7} 37.1432 &\cellcolor[HTML]{E8F0F7}	282.1463 & \cellcolor[HTML]{E8F0F7}5.7939 &\cellcolor[HTML]{E8F0F7} 0.4896 &\cellcolor[HTML]{E8F0F7} 96.9824 &\cellcolor[HTML]{E8F0F7} 53.2392 &\cellcolor[HTML]{E8F0F7} 2332.12 &\cellcolor[HTML]{E8F0F7} 	0.5636 &\cellcolor[HTML]{E8F0F7}	5.4440 \\
\hline\hline

\multirow{20}{*}{\rotatebox[origin=c]{90}{{\boldmath\textbf{WDC dataset $(256\times 256\times 191)$}}}} 
&\multirow{4}{*}{\parbox{2cm}{Case 1 \\ \small G:0.1+S:0.15}}
& PSNR $\uparrow$ & 12.3548 & 24.6747 & 33.0144 & 33.9027 & 27.7828  & 35.1606  & 34.9199 & \textbf{35.3021} & 25.1733 & \underline{35.2163} \\
&& SSIM $\uparrow$ & 0.1543 & 0.7675 & 0.9398 & 0.9454 & 0.8747  & \underline{0.9547} & 0.9534  & \textbf{0.9574}  & 0.7865  & 0.9522 \\
&& SAM $\downarrow$& 0.7534 & 0.1998  & 0.0679 & 0.0805 & 0.2195 & 0.0750  & 0.0712 & \textbf{0.0580} & 0.1877  & \underline{0.0652} \\
&& \cellcolor[HTML]{E8F0F7}Time & \cellcolor[HTML]{E8F0F7}-- &\cellcolor[HTML]{E8F0F7} 57.2959 & \cellcolor[HTML]{E8F0F7}328.3862 & \cellcolor[HTML]{E8F0F7}7.3394 & \cellcolor[HTML]{E8F0F7}0.3624 & \cellcolor[HTML]{E8F0F7}150.2327 & \cellcolor[HTML]{E8F0F7}99.1206 & \cellcolor[HTML]{E8F0F7}3140.28 & \cellcolor[HTML]{E8F0F7}1.0527 & \cellcolor[HTML]{E8F0F7}5.7252 \\
\cline{2-13}  
&\multirow{4}{*}{\parbox{2cm}{Case 2 \\ \small G:0.1+D:0.2}}
& PSNR $\uparrow$ & 17.1362 & 25.1720 & 28.0213 & 28.1949 & 25.8987 & \underline{28.4545}  & 27.7349  & 26.3535 & 25.9312 & \textbf{29.2613} \\
&& SSIM $\uparrow$ & 0.2912  & 0.7770  & 0.8838 & \underline{0.9010} & 0.8859 & 0.8902 & 0.8943  & 0.8586  & 0.8820  & \textbf{0.9114} \\
&& SAM $\downarrow$ & 0.6340 & 0.1595 & 0.1189 & 0.1055 & \underline{0.0992}  & 0.1150  & 0.1168 &  0.1254  & 0.1076 & \textbf{0.0959} \\
& & \cellcolor[HTML]{E8F0F7}Time & \cellcolor[HTML]{E8F0F7}-- &\cellcolor[HTML]{E8F0F7} 57.7331 &\cellcolor[HTML]{E8F0F7} 302.3589 &\cellcolor[HTML]{E8F0F7} 7.1882 & \cellcolor[HTML]{E8F0F7}0.3957 & \cellcolor[HTML]{E8F0F7}121.7640 & \cellcolor[HTML]{E8F0F7}102.6495 &\cellcolor[HTML]{E8F0F7} 3496.92 & \cellcolor[HTML]{E8F0F7}1.0226 & \cellcolor[HTML]{E8F0F7} 9.7274 \\
\cline{2-13}  
&\multirow{4}{*}{\parbox{2.35cm}{Case 3 \\ \small G:0.1+S:0.2+D:0.2}}
& PSNR $\uparrow$ & 11.5483 & 26.4651 & 26.3102 & 27.3099 & \underline{27.9142} & 27.3207 & 27.8952 & 26.1975 & 26.8349  & \textbf{28.9983} \\
&& SSIM $\uparrow$ & 0.1106 & 0.7472 & 0.8085 & 0.8531 & 0.8530  & 0.8446 & \underline{0.8690}  & 0.8377 & 0.7838 & \textbf{0.8749} \\
&& SAM $\downarrow$& 0.8674 & 0.1767 & \textbf{0.1087} & 0.1292 & 0.1537  & 0.1362 & 0.1180  & 0.1203 & 0.1775  & \underline{0.1168} \\
& & \cellcolor[HTML]{E8F0F7}Time &\cellcolor[HTML]{E8F0F7} -- &\cellcolor[HTML]{E8F0F7} 57.8441 & \cellcolor[HTML]{E8F0F7}300.2184 &\cellcolor[HTML]{E8F0F7} 7.3095 & \cellcolor[HTML]{E8F0F7}0.3696 & \cellcolor[HTML]{E8F0F7}126.4709 & \cellcolor[HTML]{E8F0F7}102.6412  &\cellcolor[HTML]{E8F0F7} 3443.24 & \cellcolor[HTML]{E8F0F7}1.0185 & \cellcolor[HTML]{E8F0F7}7.5857 \\
\cline{2-13}  
&\multirow{4}{*}{\parbox{2.35cm}{Case 4 \\ \small G:0.2+S:0.1+D:0.1}}
& PSNR $\uparrow$ & 11.4244 & 23.5148 & 25.1901 & 26.1204 & 22.9700 & \underline{26.3263}  & 25.6969 & 24.3378 & 24.0866 & \textbf{27.0675} \\
&& SSIM $\uparrow$ & 0.1027  & 0.6254 & 0.8275 & \underline{0.8582} & 0.7287& 0.8382  & 0.8433   & 0.8067  & 0.7125 & \textbf{0.8650} \\
&& SAM $\downarrow$& 0.8934 & 0.2407  & 0.1774 & \underline{0.1261} & 0.2910 & 0.1446 & 0.1481   & 0.1441  & 0.2376  & \textbf{0.1197} \\
& & \cellcolor[HTML]{E8F0F7}Time & \cellcolor[HTML]{E8F0F7}--  &\cellcolor[HTML]{E8F0F7} 59.0010 & \cellcolor[HTML]{E8F0F7}299.8019 & \cellcolor[HTML]{E8F0F7}7.1369 & \cellcolor[HTML]{E8F0F7}0.3720 & \cellcolor[HTML]{E8F0F7}123.6273 & \cellcolor[HTML]{E8F0F7}102.0561  & \cellcolor[HTML]{E8F0F7}3449.31 & \cellcolor[HTML]{E8F0F7}1.0395  & \cellcolor[HTML]{E8F0F7}7.5995 \\
\cline{2-13}  
&\multirow{4}{*}{\parbox{2.35cm}{Case 5 \\ \footnotesize G:0.15+S:0.15+D:0.15}} 
& PSNR $\uparrow$ & 11.6066  & 24.9023 & 25.3711 & 26.3929 & 26.5545 & 26.3706  & \underline{26.5930}  &  24.8134 & 25.4346 & \textbf{27.8692} \\
&& SSIM $\uparrow$ & 0.1095  & 0.6874 & 0.8173 & 0.8527 & 0.8333  & 0.8300 & \underline{0.8584} & 0.8114  & 0.7535 & \textbf{0.8704} \\
&& SAM $\downarrow$ & 0.8714  & 0.2085 & 0.1300 & 0.1286 & 0.1711 & 0.1442 & \underline{0.1271}   & 0.1343	 & 0.2073  & \textbf{0.1188} \\
& & \cellcolor[HTML]{E8F0F7}Time & \cellcolor[HTML]{E8F0F7}--  &\cellcolor[HTML]{E8F0F7} 55.0380 &\cellcolor[HTML]{E8F0F7} 328.5056 &\cellcolor[HTML]{E8F0F7} 7.3243 &\cellcolor[HTML]{E8F0F7} 0.3678 &\cellcolor[HTML]{E8F0F7} 157.6777 &\cellcolor[HTML]{E8F0F7} 100.2687 & \cellcolor[HTML]{E8F0F7}3325.26	 & \cellcolor[HTML]{E8F0F7}1.0483 &	\cellcolor[HTML]{E8F0F7}7.5320\\
\Xhline{1.25pt}
\end{tabular}}
\end{table*}

\subsection{Simulated Experiments}
To quantitatively assess the denoising performance of the proposed FRHD method, we first perform experiments on simulated hyperspectral datasets with ground truth (GT). Specifically, one multispectral image (MSI) from the CAVE dataset and two HSIs from the PaC and WDC datasets are adopted, under five distinct noise scenarios, to validate the performance of the FRHD model.

\emph{1) Experimental Setting}: The three widely-used benchmark simulation datasets are as follows.
\begin{itemize}
   \item \textbf{CAVE Glass Tiles Dataset}\footnote{\url{https://cave.cs.columbia.edu/repository/Multispectral}}: The glass tiles dataset was collected by the Columbia Imaging and Vision Laboratory (CAVE) at Columbia University. It provides full spectral resolution reflectance data from 400–700 nm with 10 nm intervals. The original dataset size is $512 \times 512 \times 31$. For computational efficiency, we cropped the dataset to $200 \times 200 \times 31$.
   \item \textbf{Harvard PaC Dataset}\footnote{\url{https://www.ehu.eus/ccwintco/index.php/Hyperspectral_Remote_Sensing_Scenes}}: The PaC dataset was acquired by the ROSIS sensor during a flight campaign over Pavia in northern Italy. The original dataset size is $610 \times 340 \times 103$. For visualization convenience and improved computational efficiency, we cropped it to $300 \times 300 \times 103$.
   \item \textbf{Washington DC Mall (WDC) Dataset}\footnote{\url{https://engineering.purdue.edu/~biehl/MultiSpec/hyperspectral.html}}: The WDC dataset, collected by the Virginia Center for Spectral Information Technology Applications using the HYDICE sensor, has dimensions of $256\times 256\times 191$.
\end{itemize}

Additionally, to evaluate the robustness of our algorithm, we designed five representative noise scenarios, defined as follows:
\renewcommand{\labelitemi}{-}  
\begin{itemize}
    \item Case 1: A mixture of Gaussian noise ($\mu=0$, $\sigma^2=0.1$) and $15\%$ sparse salt-and-pepper noise.
    \item Case 2: Gaussian noise ($\mu=0$, $\sigma^2=0.1$) combined with $20\%$ deadline noise.
    \item Case 3: Gaussian noise ($\mu=0$, $\sigma^2=0.2$) in conjunction with $10\%$ sparse salt-and-pepper and deadline noise.
    \item Case 4: Gaussian noise ($\mu=0$, $\sigma^2=0.1$) with $20\%$ sparse salt-and-pepper and deadline noise.
    \item Case 5: Gaussian noise ($\mu=0$, $\sigma^2=0.15$) combined with $15\%$ sparse salt-and-pepper and deadline noise.
\end{itemize}

Cases 1 and 2 consist of Gaussian noise combined with single salt-and-pepper or deadline noise, while Cases 3 to 5 simulate more complex mixtures of Gaussian, salt-and-pepper, and deadline noise at varying intensity levels. These configurations effectively represent the diverse and challenging noise patterns commonly encountered in real-world HSI tasks. Fig.~\ref{noise_case} illustrates the impact of these noise cases.

\emph{2) Comparison indexes}: We employ four quantitative metrics to evaluate the denoising performance of different methods: Peak Signal-to-Noise Ratio (PSNR), Structural Similarity Index (SSIM), and Spectral Angle Mapper (SAM). PSNR and SSIM primarily assess spatial information, with higher values indicating better image quality after denoising. In contrast, SAM evaluates the preservation of spectral information, where lower values reflect superior restoration performance. Additionally, we also provide the running time (in seconds) for all competing algorithms as a reference.

\emph{3) Quantitative Analysis of results}: Table~\ref{tab:simulated} presents the quantitative results on three simulated datasets (CAVE, PaC, and WDC) under various noise scenarios. In addition, Fig.~\ref{Simulated_results} further visualizes the reconstructed pseudo-RGB images under Case~5, along with the corresponding error maps, providing intuitive insights into the denoising effectiveness of each method.

As shown in Table \ref{tab:simulated}, the proposed FRHD method outperforms the competing methods across various test datasets and noise scenarios in most cases. Our method demonstrates exceptional performance even under conditions of enhanced mixed noise, a trend that is further reflected in Fig.~\ref{Simulated_results}. Among the competing methods, certain approaches demonstrate impressive performance in specific scenarios—for example, FastHyMix and FBGND, which incorporate both network- and model-based priors. However, their overall performance may be slightly worse than our method across diverse noise conditions. Notably, FastHyMix is fast because it uses a pre-trained FFDNet and requires no training, whereas FBGND is slower due to additional training and the absence of GPU acceleration. As the number of spectral bands increases, the advantages of our method in both computational efficiency and recovery quality become more pronounced, particularly on the high-dimensional PaC and WDC datasets. This efficiency stems from the integration of subspace representation with the RCTV regularizer, which decouples the high-dimensional fidelity term. For instance, on the PaC dataset ($300 \times 300 \times 103$), our method requires only about 5 seconds.

Overall, the proposed FRHD method exhibits robust performance under diverse mixed noise conditions and consistently outperforms compared methods in terms of both runtime efficiency and denoising quality.

\begin{figure}[!t]
\centering
\includegraphics[width=5.9in]{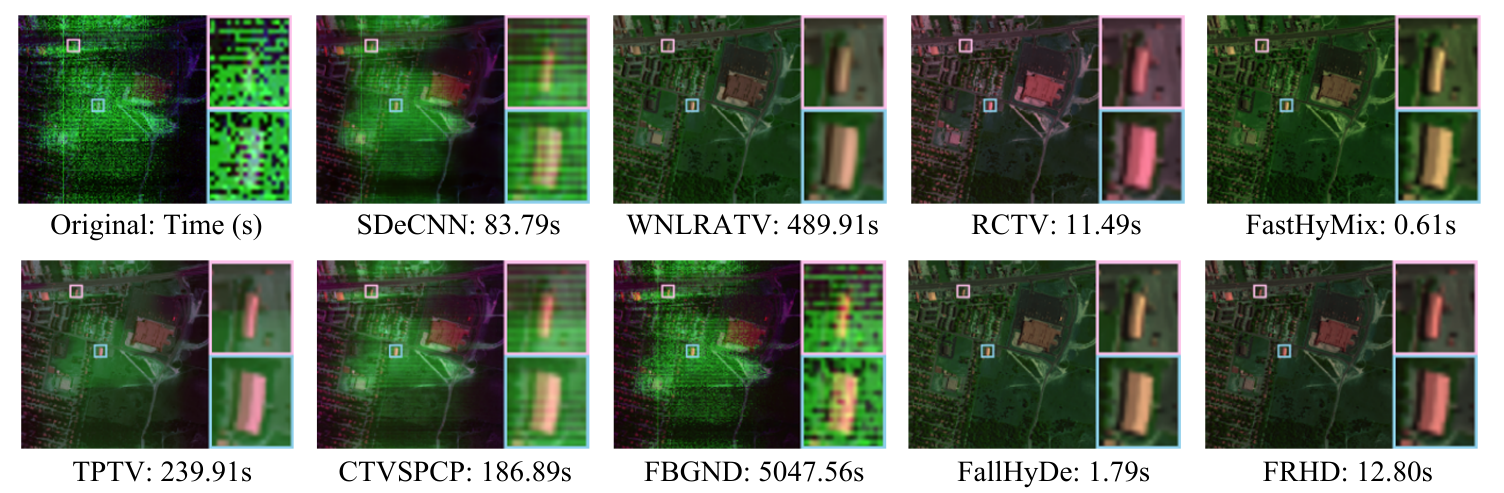}
\caption{Denoising results of all competing methods on the HYDICE Urban city dataset. The R-G-B channels of the color image are composed of spectral bands 2, 150, and 206, respectively.}
\label{urban1}
\end{figure}

\subsection{Real Data Experiments}
To further validate the effectiveness of the proposed FRHD model, we tested its performance on HSI denoising using the well-known real-world HYDICE Urban\footnote{\url{https://www.agc.army.mil/What-we-do/Hypercube}} dataset. This dataset is a widely used benchmark in HSI denoising experiments and is affected by various types of noise, including Gaussian noise, stripe noise, and other unspecified disturbances. The original size of the dataset is $307 \times 307 \times 210$.

Real HSI denoising is more challenging due to the inherent complexity and unknown nature of noise patterns in real-world data. Performance evaluation on real datasets using traditional quantitative metrics is impractical without GT. Fortunately, we can qualitatively assess the denoising effect by visualizing the reconstructed image as a pseudo-RGB representation of the corrupted bands.

Fig. \ref{urban1} presents the pseudo-RGB image of the Urban dataset. As shown in Fig. \ref{urban1}, the proposed method demonstrates excellent denoising performance on the HYDICE Urban dataset, effectively reducing noise while preserving key texture details. It is worth noting that among the compared methods, SDeCNN and FBGND fail to effectively suppress the green background, while the remaining approaches can only partially mitigate its effect. Although WNLRATV, RCTV, and FastHyMix achieve satisfactory denoising results, they still suffer from noticeable color distortion, often introducing greenish or reddish background artifacts. In contrast, our method not only effectively removes various types of noise, but also produces contours that are closer to the ideal true representation. Therefore, the tests on real datasets further validate the effectiveness of the proposed algorithm.

In summary, the proposed FRHD method demonstrates excellent performance and computational efficiency on both simulated and real-world datasets.

\begin{figure}[!t]
\centering
\subfloat[CAVE dataset]{\includegraphics[width=1.65in]{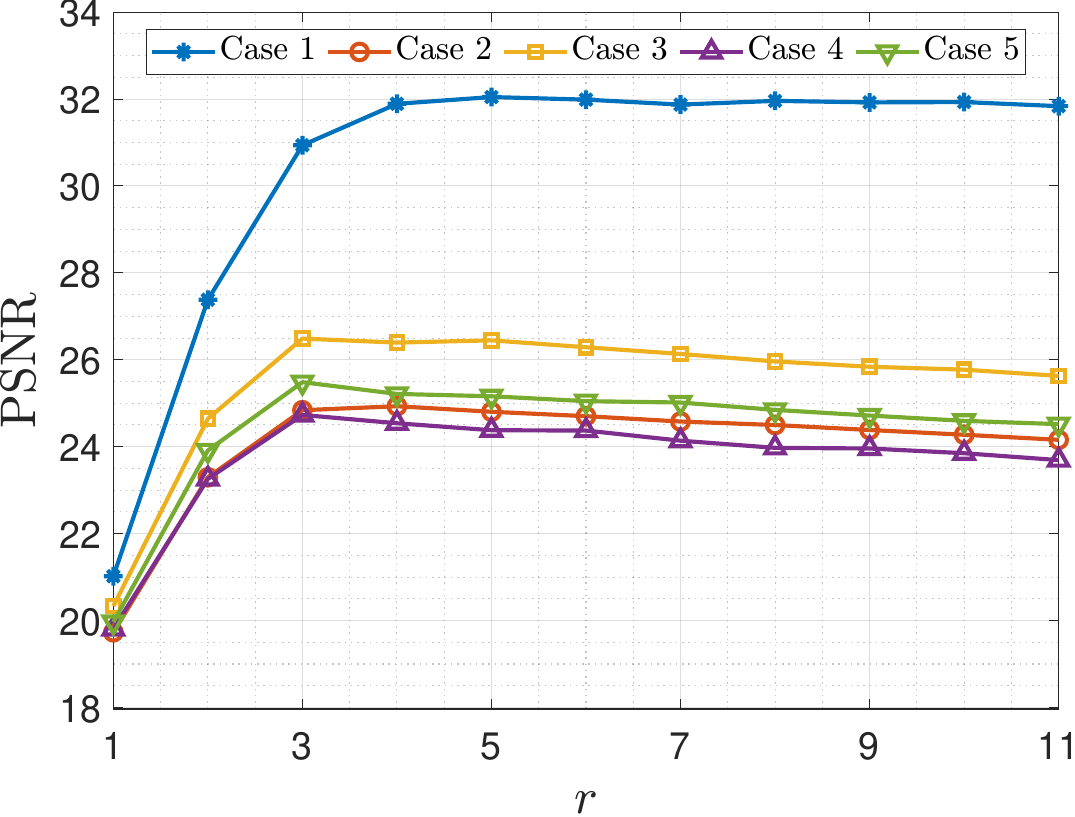}%
\label{rank_para_a}}
\hspace{1.5em}
\subfloat[PaC dataset]{\includegraphics[width=1.65in]{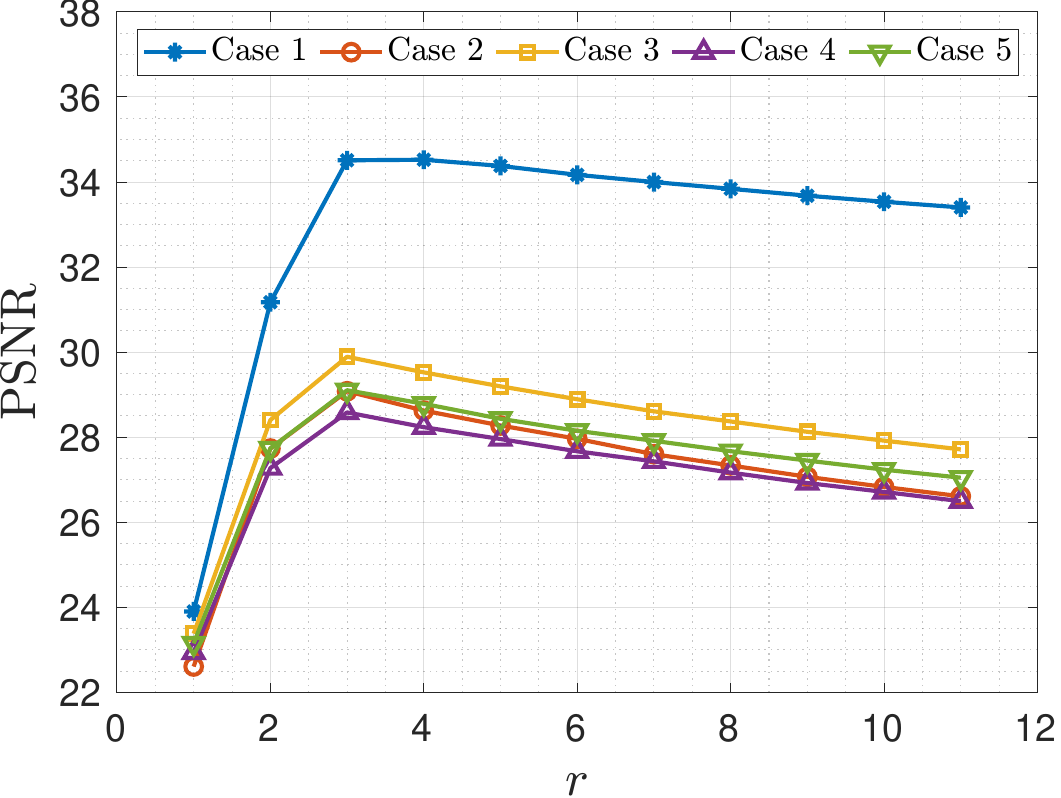}%
\label{rank_para_b}}
\hspace{1.5em}
\subfloat[WDC dataset]{\includegraphics[width=1.65in]{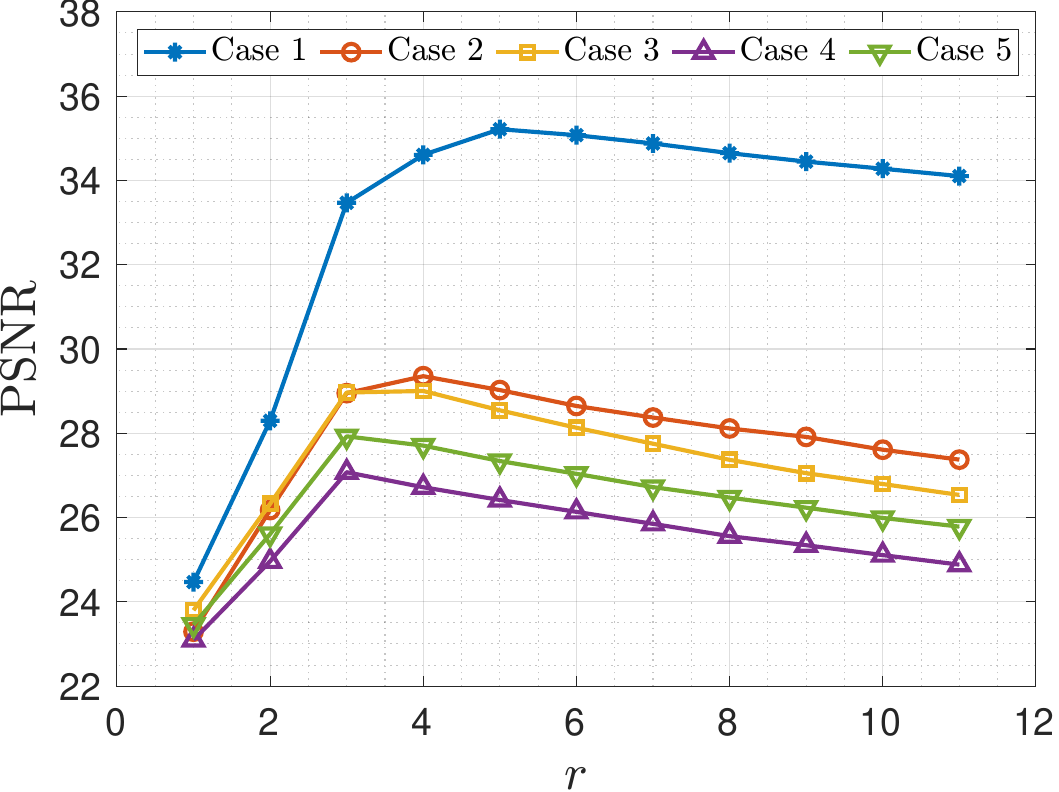}%
\label{rank_para_c}}
\caption{{Sensitivity analysis of FRHD model parameter $r$ under different noise conditions for three datasets.}}
\label{param_r_sensitivity}
\end{figure}

\begin{figure}[!t]
\centering
\subfloat[CAVE dataset]{\includegraphics[width=1.65in]{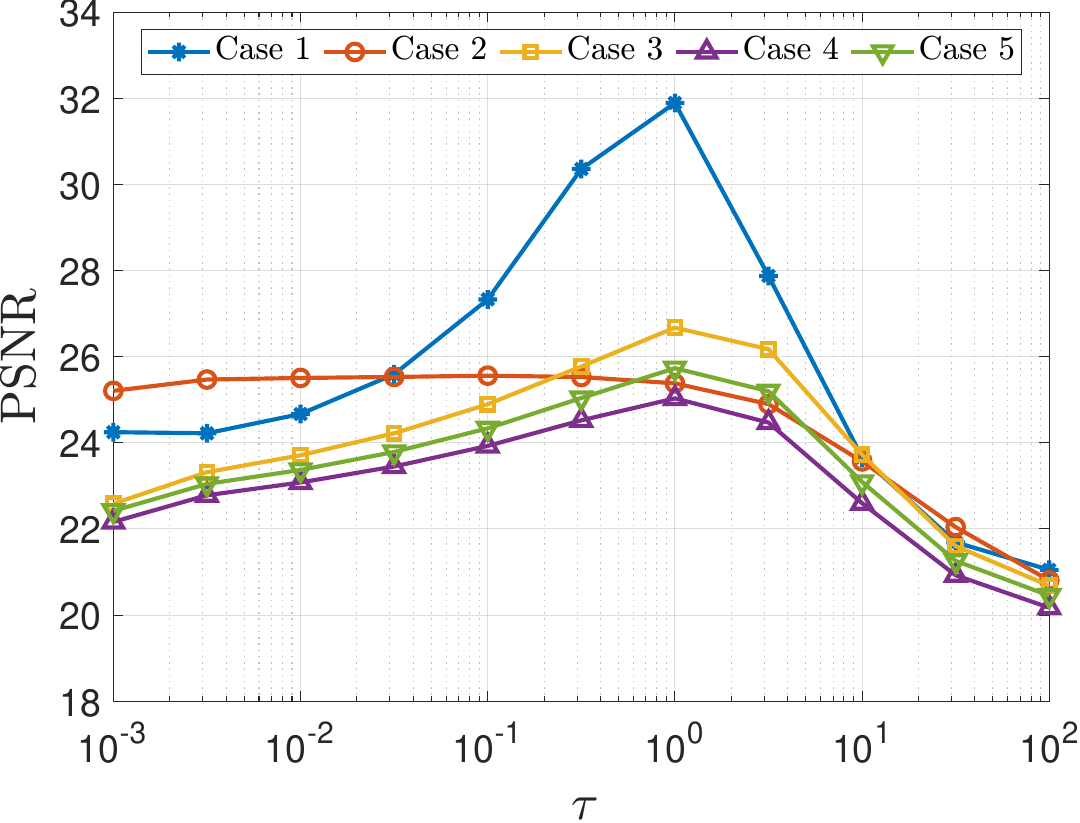}%
\label{tau_para_a}}
\hspace{1.5em}
\subfloat[PaC dataset]{\includegraphics[width=1.65in]{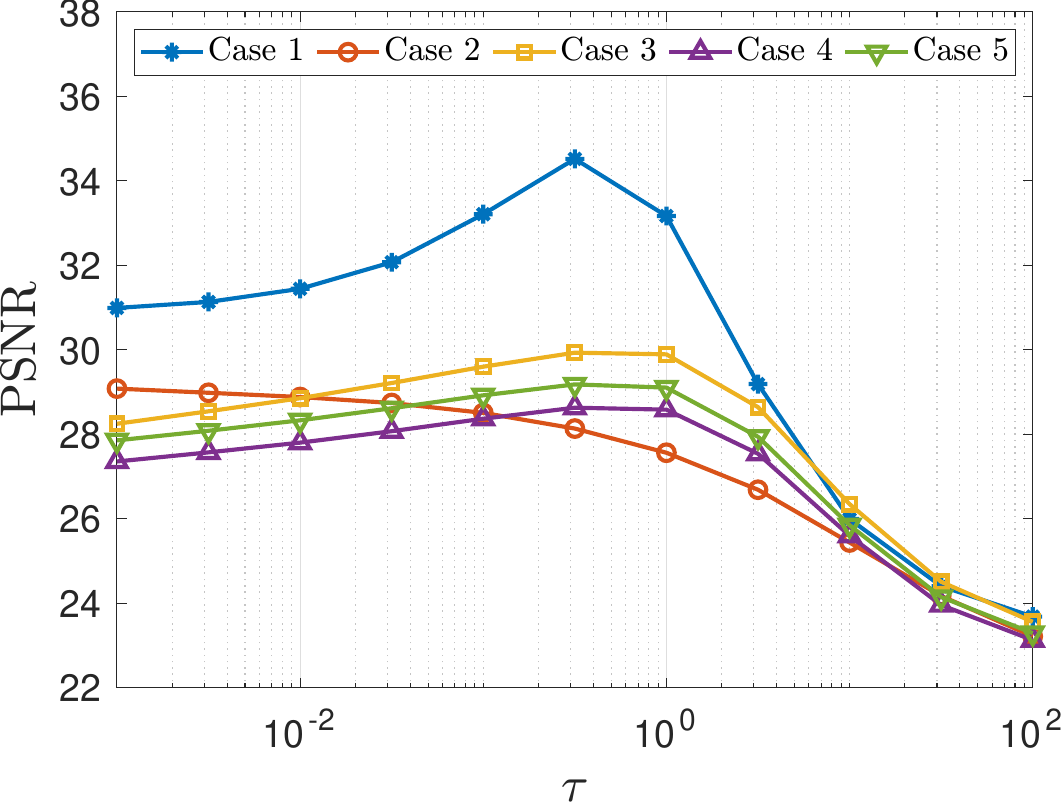}%
\label{tau_para_b}}
\hspace{1.5em}
\subfloat[WDC dataset]{\includegraphics[width=1.65in]{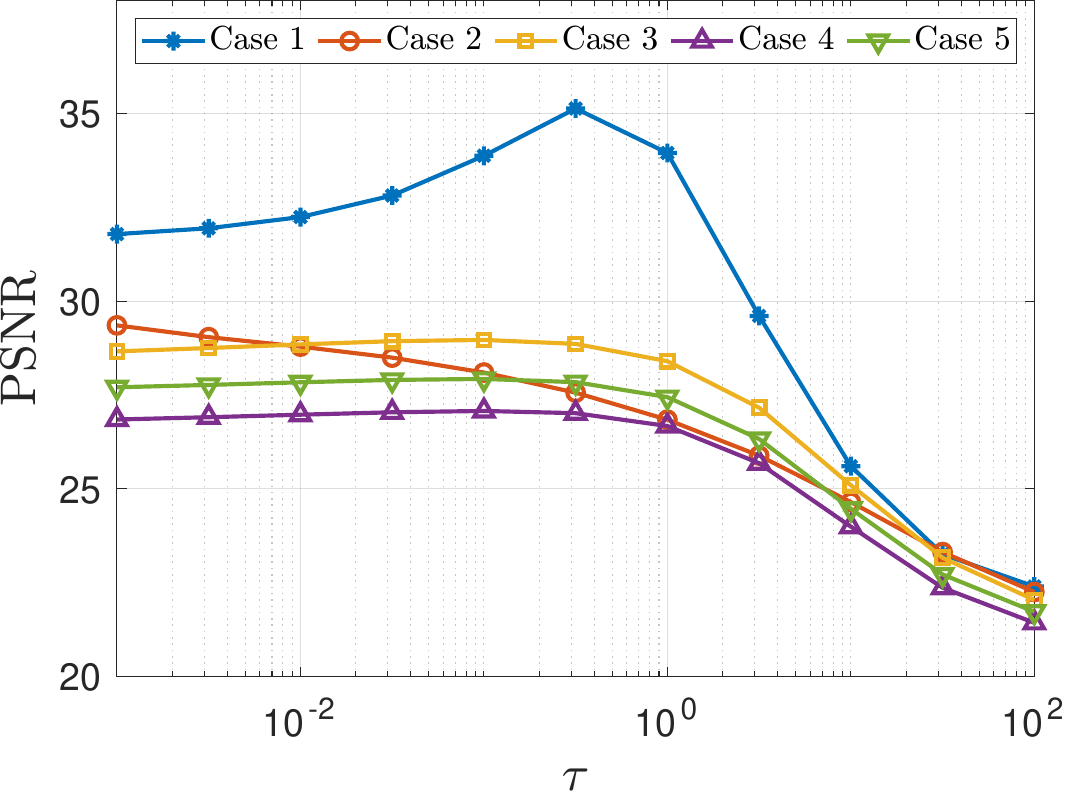}%
\label{tau_para_c}}
\caption{Sensitivity analysis of FRHD model parameter $\tau$ under different noise conditions for three datasets.}
\label{param_tau_sensitivity}
\end{figure}

\begin{figure*}[!t]
\centering
\subfloat[]{\includegraphics[width=1.23in]{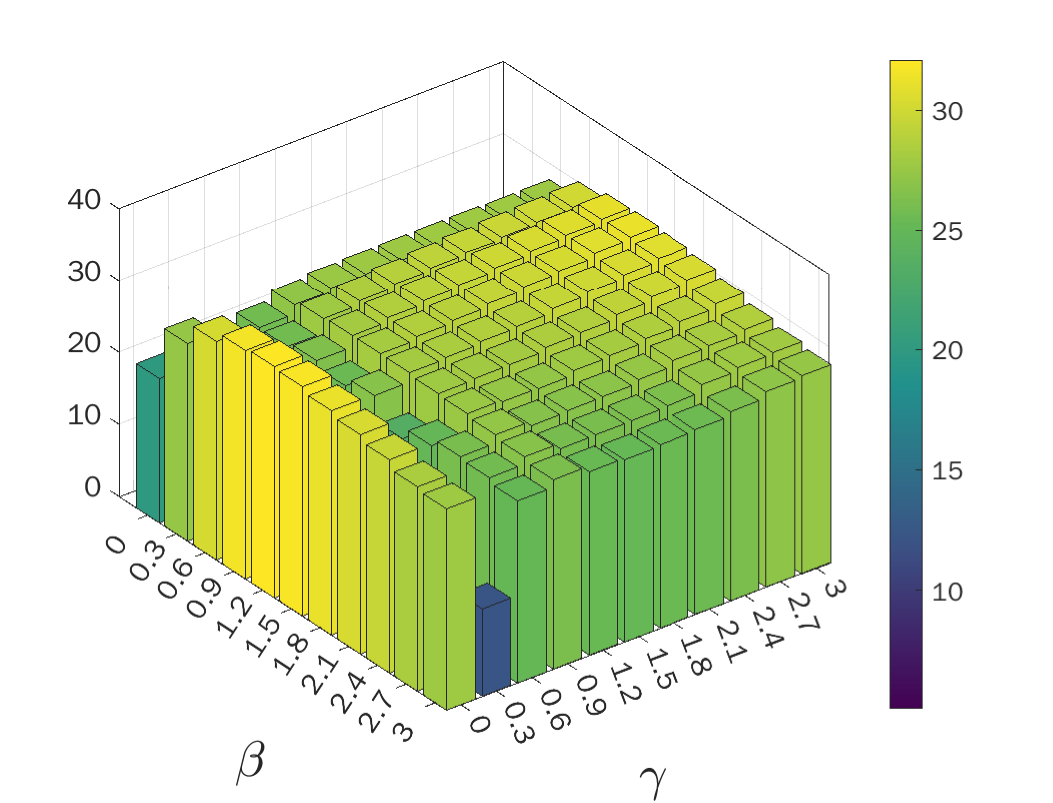}%
\label{CAVE,case1}}
\hfil
\subfloat[]{\includegraphics[width=1.23in]{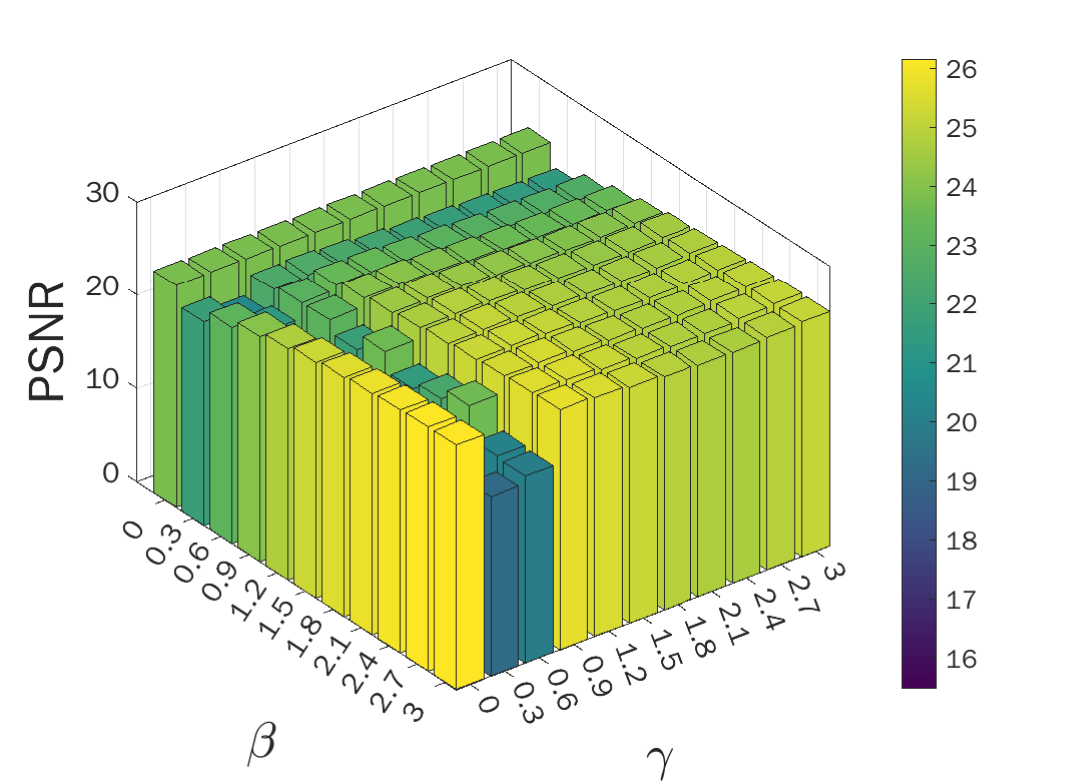}%
\label{CAVE,case2}}
\hfil
\subfloat[]{\includegraphics[width=1.23in]{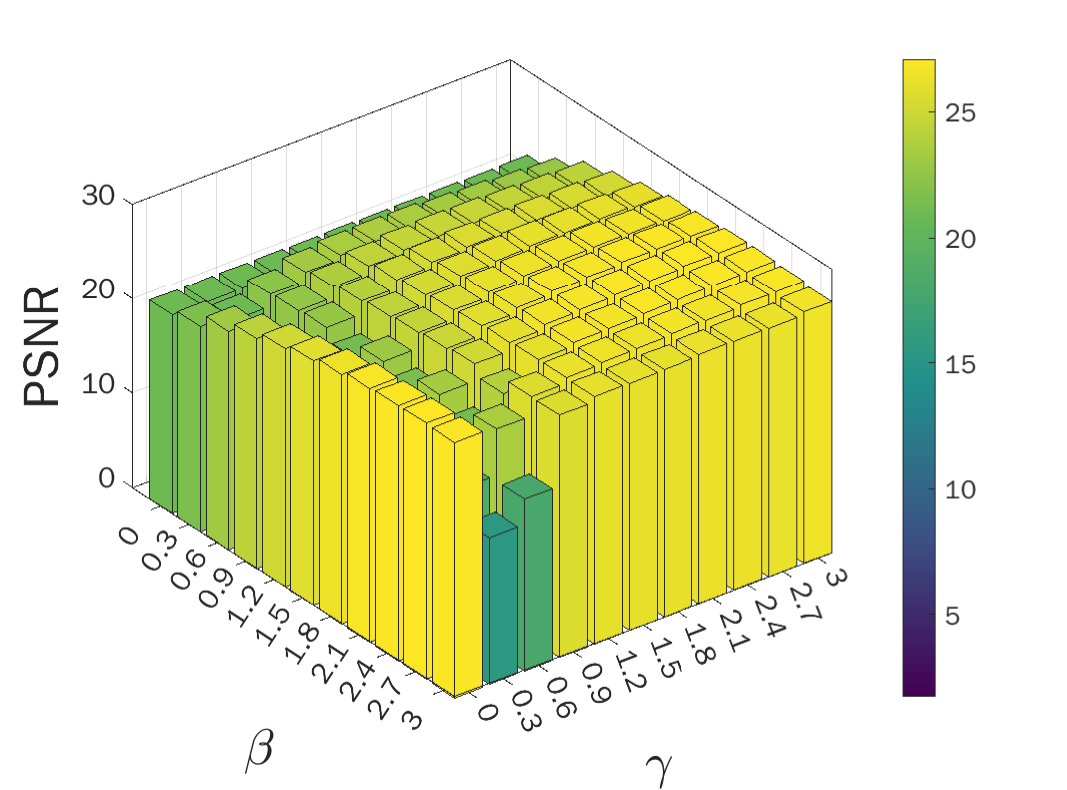}%
\label{CAVE,case3}}
\hfil
\subfloat[]{\includegraphics[width=1.23in]{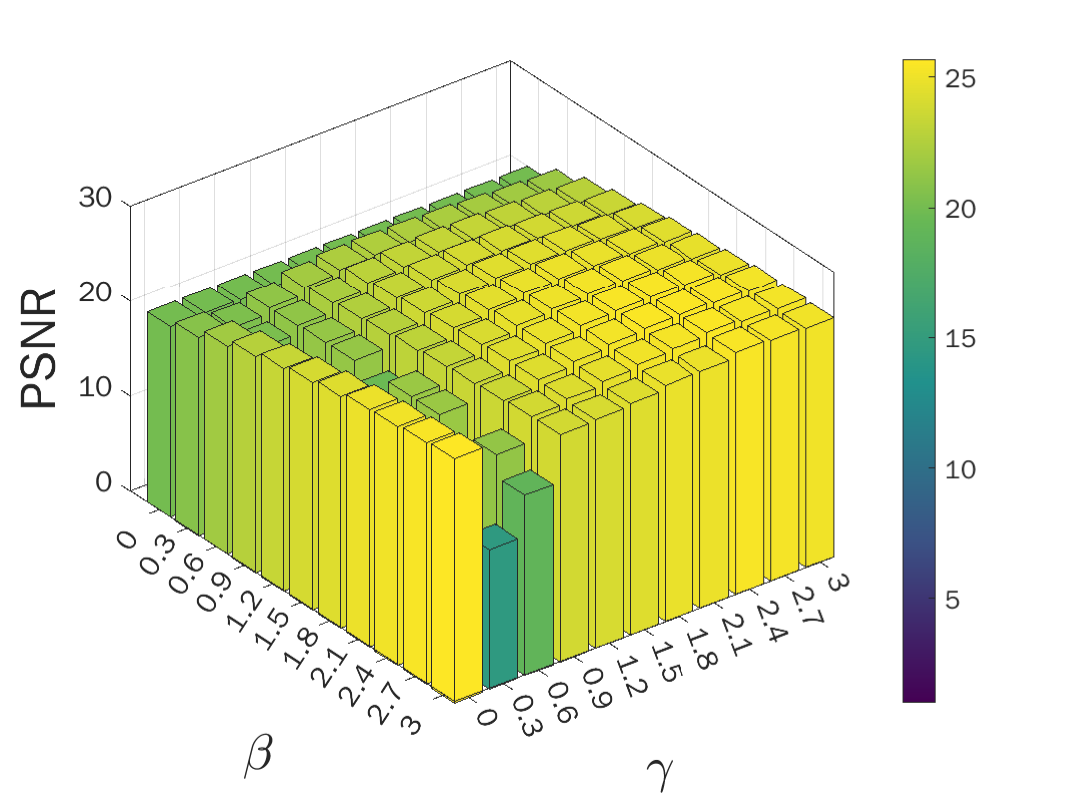}%
\label{CAVE,case4}}
\hfil
\subfloat[]{\includegraphics[width=1.23in]{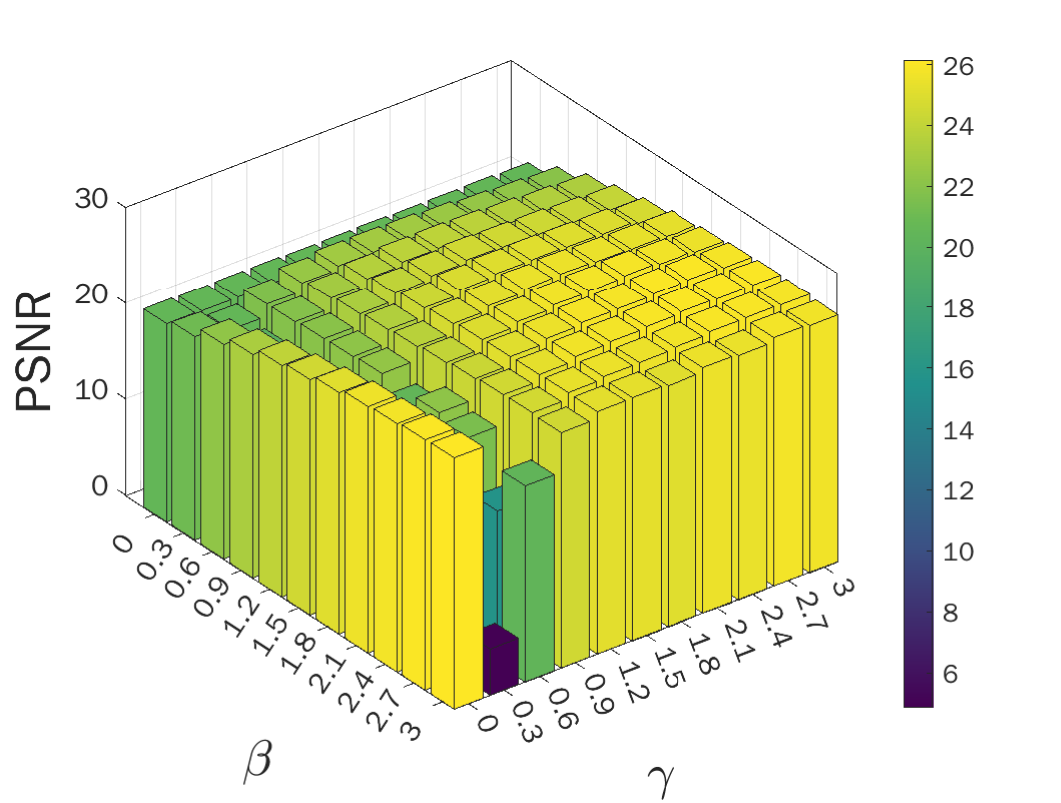}%
\label{CAVE,case5}}
\caption{Sensitivity of FRHD model parameters $\beta$ and  $\gamma$ to CAVE dataset under different noise conditions. (a) Case 1; (b) Case 2; (c) Case 3; (d) Case 4; (e) Case 5.}
\label{CAVE_Weight_para}
\end{figure*}

\subsection{Parameter analysis and setting rationale}

This subsection analyzes the impact of key hyperparameters and explains the rationale behind their selection. The proposed FRHD method requires four main parameters: the SVD rank $r$ in RCTV, the fidelity weight $\tau$, and the regularization coefficients $\beta$ and $\gamma$ for sparse and directional noise, respectively. Our parameter selection follows a principled approach based on HSI characteristics and noise properties.

\subsubsection{Rationale for initial parameter selection}

a) SVD Rank $r$: This parameter governs the degree of spectral low-rank approximation in the RCTV model. The value of $r$ is selected based on the intrinsic spectral dimensionality of hyperspectral images (HSIs), which is typically constrained by strong inter-band correlations. Empirical observations indicate that most HSIs have an effective spectral rank between 3 and 10. Accordingly, we initialize $r = 2$ and increment it iteratively until denoising performance begins to degrade.

b) Fidelity Weight $\tau$: This parameter controls the trade-off between data fidelity and regularization. Following classical regularization parameter selection principles, we first define a coarse search space using a geometric progression (e.g., ${0.001, 0.01, 0.1, 1, 10}$), then perform refined local search within promising intervals to determine the optimal value.

c) Regularization Coefficients $\beta$ and $\gamma$: These parameters respectively govern the regularization strength for sparse noise and directional noise components. Their selection incorporates physical interpretation: $\beta$ is scaled proportionally to the estimated sparse noise density, while $\gamma$ relates to the estimated stripe intensity. For mixed-noise scenarios, we initialize $\beta = \gamma = 0.1$ based on empirical observations from preliminary experiments, then conduct systematic grid search for final tuning. Notably, in cases containing only non-directional noise (Case 1), we explicitly set $\gamma = 0$ to disable the corresponding regularization term.

\subsubsection{Sensitivity and generality across different HSI datasets}

Regarding parameter sensitivity, among the four tunable hyperparameters, the SVD rank $r$ and the fidelity weight $\tau$ may be influenced by variations in hyperspectral image characteristics. We first analyze the sensitivity of these two parameters across diverse HSI datasets. Experiments are conducted on three widely used benchmarks: CAVE (31 bands, $200 \times 200$), PaC (103 bands, $300 \times 300$), and WDC (191 bands, $256 \times 256$), which differ substantially in spectral bands and spatial resolution. The corresponding sensitivity results are illustrated in Fig.s \ref{param_r_sensitivity} and \ref{param_tau_sensitivity}. It can be observed that both $r$ and $\tau$ show minimal variation across datasets of different sizes and band counts. Moreover, under different noise scenarios, the optimal parameter values follow nearly consistent trends, indicating their relative insensitivity to data variations.

As for the regularization coefficients $\beta$ and $\gamma$, which control the sparse and directional noise priors respectively, they are not strongly correlated with dataset scale but may depend on noise characteristics. We thus evaluate their sensitivity under different noise types using the CAVE dataset, with results shown in Fig. \ref{CAVE_Weight_para}. Despite varying noise conditions, stable denoising performance is attained when $\beta \in [0.9, 2.5]$ and $\gamma \in [0, 2.5]$, demonstrating the robustness of these parameters.

For practical guidance, we summarize the tested noise configurations and corresponding parameter settings across datasets in Table~\ref{tab:para_setting}. Both empirical observations and experimental results confirm that the proposed method exhibits low parameter sensitivity and is straightforward to configure in practice.

\subsubsection{Discussion on adaptive parameter selection}

While the proposed FRHD model demonstrates strong robustness to parameter variations under the current manual tuning framework, developing adaptive parameter selection strategies could further improve its accessibility for non-expert users. Future research directions for automating parameter estimation within the FRHD framework include:
\begin{itemize}
    \item Automatically estimating the TV weight $\tau$ based on the global noise variance $\hat{\sigma}^2$ derived from the input data.
    \item Developing heuristic rules to initialize the regularization coefficients $\beta$ and $\gamma$ based on preliminary noise analysis (e.g., through statistical estimation or fast filtering techniques).
    \item Exploring optimization-based approaches, such as bilevel optimization, to learn parameter selection policies from representative HSI datasets.
\end{itemize}

The integration of such adaptive mechanisms would enhance the practicality of the framework and advance toward a more autonomous HSI denoising pipeline.

\begin{table}[!t]
\fontfamily{ptm}\selectfont 
\centering
\caption{Parameter settings for CAVE, PaC, and WDC datasets across different cases.}\label{tab:para_setting}
\renewcommand{\arraystretch}{1.05}
\resizebox{0.6\columnwidth}{!}{ 
\begin{tabular}{|l|c|c|c|c|c|c|c|c|c|}
\hline
\multirow{2}{*}{\textbf{Cases}} & \multicolumn{3}{c|}{\textbf{CAVE}} & \multicolumn{3}{c|}{\textbf{PaC}} & \multicolumn{3}{c|}{\textbf{WDC}} \\ 
\cline{2-4} \cline{5-7} \cline{8-10}
 & $r$ & $\tau$ & $[\beta, \gamma]$ & $r$ & $\tau$ & $[\beta, \gamma]$ & $r$ & $\tau$ & $[\beta, \gamma]$ \\ 
\hline
Case 1   & 4 & 1   & [1, 0]     & 3 & 0.5   & [0.95, 0]  & 5 & 0.47  & [1.05, 0] \\
\hline
Case 2   & 4 & 3   & [2.1, 1]   & 3 & 0.001 & [1.3, 0.7] & 4 & 0.001 & [1.3, 0.7] \\
\hline
Cases 3-5 & 3 & 2.5 & [1.5, 2.5] & 3 & 1     & [1.5, 2.5] & 3 & 0.1   & [1.5, 2.5] \\
\hline
\end{tabular} }
\end{table}

\begin{table}[!t]
\fontfamily{ptm}\selectfont 
\caption{Denoising performance comparison of each module in the proposed FRHD against the Baseline model. \label{tab:compar}}
\centering
\setlength{\arrayrulewidth}{1pt} 
\renewcommand{\arraystretch}{1.45}
\resizebox{0.65\columnwidth}{!}{
\Huge
\begin{tabular}{|l|cccc|cccc|cccc|}
\hline
\multicolumn{1}{|c|} {\textbf{Datasets}} & \multicolumn{4}{c}{\textbf{CAVE}} & \multicolumn{4}{|c|}{\textbf{PaC}} & \multicolumn{4}{c|}{\textbf{WDC}} \\ 
\hline
Ablation Module & PSNR & SSIM & SAM & Time & PSNR & SSIM& SAM & Time & PSNR &SSIM&  SAM & Time\\ 
\hline
Noisy & 11.6046 &  0.0836 & 0.9221 & -- & 11.7823 & 0.0724 & 0.9357 & -- & 11.6066 & 0.1095 & 0.8714 & -- \\
\hline
{RCTV} & {25.0842} & {0.7789} & {0.1997} & {0.9060} & 
{26.1311} & {0.7786} & {0.1555} & {5.7939} & 
{26.3929} & {0.8527}  & {0.1286} & {7.3243} \\
\hline
Baseline & 23.9373 & 0.7696 & 0.1887 & 0.5655 & 25.6765 & 0.7779 & 0.1412 & 4.3299 & 25.7742 & 0.8246  & 0.1439 & 4.6729 \\
\hline
Baseline+A $\Uparrow$ & 24.8229 & 0.8172 & 0.1681 & 0.7323 & 28.0799 & 0.8215 & 0.1110 & 4.9823 & 26.7997 & 0.8563 & 0.1159 & 6.4919 \\
\hline
Baseline+A+B $\Uparrow$ & 25.7411 & 0.8242 & 0.1694 & 0.8567 & 29.1734 & 0.8275 & 0.1172 & 5.4872 & 27.8692  & 0.8704 & 0.1188 & 7.5402 \\
\hline
\end{tabular}}
\end{table}
\subsection{Ablation experiments}

Here, we systematically evaluate the contributions of the two core components in the proposed FRHD framework. As outlined in Section~\ref{sect3}, FRHD comprises two key modules: Module~A, a noise prior reduction strategy, and Module~B, an adaptive pixel-wise weighting strategy. We compare our framework against the original RCTV method and define our ablation configurations as follows. \textbf{Baseline} is a simplified denoising model incorporating RCTV regularization but excluding both noise prior reduction modeling and adaptive pixel-wise fidelity term. The Baseline differs from the original RCTV method in two key aspects: (1) it eliminates the explicit Gaussian noise term $E$ and its associated parameter $\beta$, reducing the parameter count from four to three; and (2) it employs a convergence tolerance of $\epsilon = 10^{-5}$ compared to RCTV's $\epsilon = 10^{-6}$. \textbf{Baseline+A} augments the Baseline model with $\ell_{2,1}$ regularization and noise prior reduction to address structured noise, while \textbf{Baseline+A+B} further incorporates the adaptive weighting mechanism, resulting in the complete FRHD model.

Table~\ref{tab:compar} reports the denoising performance of each configuration under Case~5. The results reveal several key observations. First, the original RCTV method achieves marginally better PSNR than our Baseline model (e.g., +0.62~dB on WDC), which can be attributed to its stricter convergence criterion ($\epsilon=10^{-6}$ versus our Baseline's $\epsilon=10^{-5}$). This stricter tolerance yields slightly improved accuracy at the cost of increased computational time. Second, incorporating Module~A yields substantial improvements of 1--3~dB in PSNR and 3--5\% in SSIM over the Baseline, demonstrating that the $\ell_{2,1}$-regularized noise prior effectively compensates for the simplified model while delivering superior performance. Third, adding Module~B provides further gains of approximately 1--2~dB in PSNR and 1--2\% in SSIM relative to Baseline+A, confirming the complementary value of the adaptive weighting mechanism.

Most significantly, the complete FRHD framework surpasses the original RCTV method by 1.48~dB, 3.04~dB, and 1.48~dB in PSNR on CAVE, PaC, and WDC datasets respectively, while maintaining comparable computational efficiency. These results clearly validate the individual and complementary contributions of both modules in enhancing the overall performance of the FRHD method. In terms of runtime, the two modules introduce minimal computational overhead. Except for the WDC dataset, the additional time cost is almost negligible, indicating the proposed model maintains favorable computational efficiency.

Overall, the complete FRHD framework achieves a notable improvement of 2--4~dB in PSNR and 5\%--6\% in SSIM compared to the baseline, while outperforming the original RCTV method with a more parsimonious parameterization and competitive runtime performance.

\subsection{Other comparative experiments}
\subsubsection{Robustness of FRHD Under Extreme Noise Conditions}
To rigorously assess the robustness of the proposed FRHD model, we performed a comprehensive suite of gradient-based experiments to evaluate its performance under progressively intensified and extreme noise conditions. Two scenarios were investigated: 

i) Salt-and-pepper noise with intensity gradually increasing from $5\%$ to $50\%$;

ii) Additive Gaussian noise with standard deviation $\sigma$ increasing from $0.05$ to $0.5$ in steps of $0.05$, combined with fixed impulse noise ($5\%$) and stripe noise ($5\%$).

The experimental results are presented in Figs.~\ref{WDC_salt_grad} and~\ref{WDC_gauss_grad}. As demonstrated, the proposed FRHD method consistently outperforms the baseline RCTV approach across all quantitative metrics (PSNR, SSIM, and SAM) under all tested conditions. Notably, this performance advantage persists even under extreme noise settings, including Gaussian noise beyond $\sigma=0.3$ and salt-and-pepper noise exceeding $30\%$. These findings confirm the robustness and practical effectiveness of FRHD for denoising HSIs contaminated by complex unknown noise.

\begin{figure}[!t]
\centering
\includegraphics[width=5.5in]{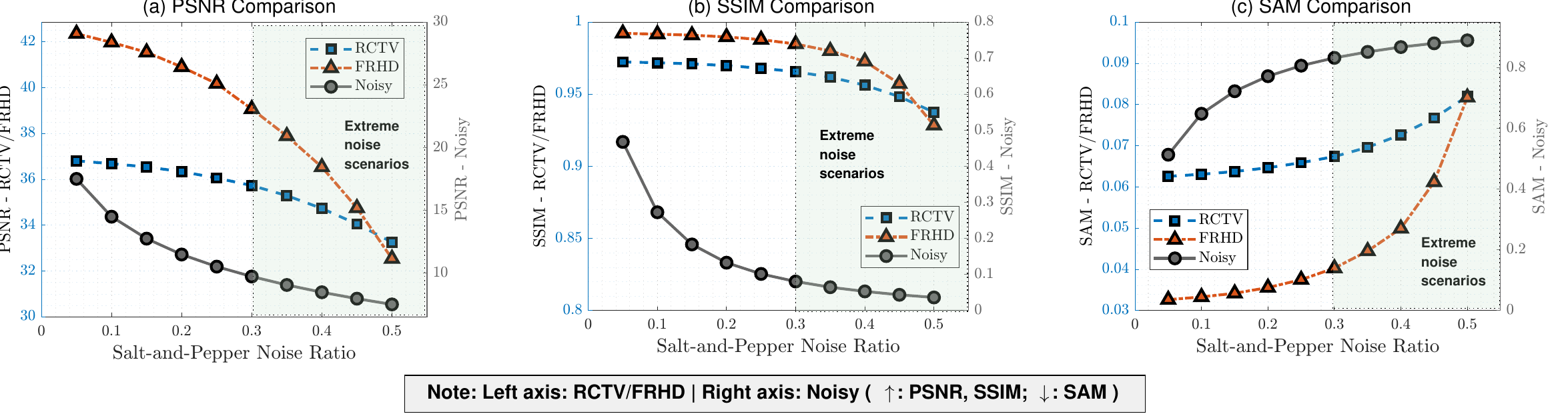}
\caption{{Gradient Experiment (Gradually Intensifying Salt-and-Pepper Noise).}}
\label{WDC_salt_grad}
\end{figure}

\begin{figure}[!t]
\centering
\includegraphics[width=5.5in]{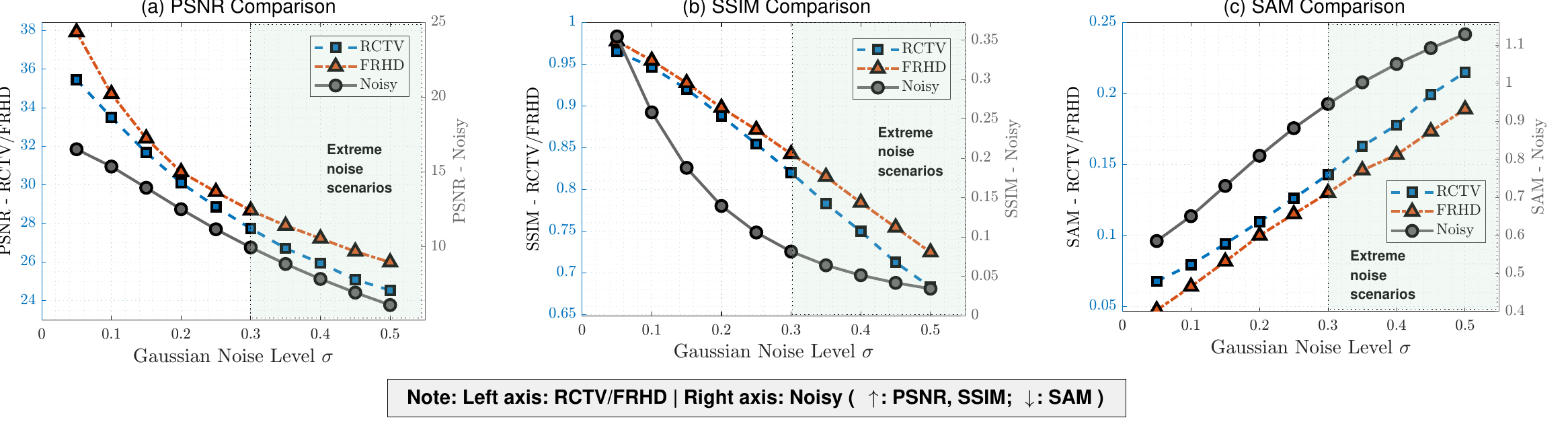}
\caption{{Gradient Experiment (5\% Salt-and-Pepper + 5\% Stripes + Gradually Intensifying Gaussian Noise).}}
\label{WDC_gauss_grad}
\end{figure}

\begin{table}[!t]
\fontfamily{ptm}\selectfont
\centering
\caption{{Comparison results on WDC dataset. Best and second-best results are highlighted in \textbf{bold} and \underline{underline}, respectively.}}
\label{inr_results}
\renewcommand{\arraystretch}{1.2}
\resizebox{0.85\columnwidth}{!}{
\begin{tabular}{c|cccc|cccc|cccc|cccc}
\hline
\multirow{2}{*}{{Cases}}
& \multicolumn{4}{c|}{{Noisy}}
& \multicolumn{4}{c|}{{HLRTF~\cite{luo2022hlrtf} }}
& \multicolumn{4}{c|}{{Flex-DLD~\cite{Chen2024flex}}}
& \multicolumn{4}{c}{{FRHD}} \\
\cline{2-17}
& {PSNR} & {SSIM} & {SAM} & {Time}
& {PSNR} & {SSIM} & {SAM} & {Time}
& {PSNR} & {SSIM} & {SAM} & {Time}
& {PSNR} & {SSIM} & {SAM} & {Time} \\
\hline
{Case 1}
& {12.36} & {0.1543} & {0.7534} & {0.00}
& {\underline{31.43}} & {\underline{0.9316}} & {\underline{0.0803}} & {\underline{17.98}}
& {31.14} & {0.9158} & {0.0811} & {89.62}
& {\textbf{35.22}} & {\textbf{0.9522}} & {\textbf{0.0652}} & {\textbf{5.73}} \\

{Case 2}
& {17.14} & {0.2912} & {0.6340} & {0.00}
& {\textbf{32.19}} & {\textbf{0.9319}} & {\textbf{0.0778}} & {\underline{17.09}}
& {27.11} & {0.8513} & {0.1247} & {89.57}
& {\underline{29.26}} & {\underline{0.9114}} & {\underline{0.0959}} & {\textbf{9.73}} \\

{Case 3}
& {11.55} & {0.1106} & {0.8674} & {0.00}
& {\underline{27.90}} & {\underline{0.8266}} & {0.1356} & {\underline{15.79}}
& {26.42} & {0.8224} & {\underline{0.1296}} & {89.49}
& {\textbf{29.00}} & {\textbf{0.8749}} & {\textbf{0.1168}} & {\textbf{7.59}} \\

{Case 4}
& {11.42} & {0.1027} & {0.8934} & {0.00}
& {\textbf{27.85}} & {\textbf{0.8501}} & {\textbf{0.1167}} & {\underline{16.38}}
& {25.05} & {0.7961} & {0.1534} & {89.33}
& {\underline{27.07}} & {\underline{0.8650}} & {\underline{0.1197}} & {\textbf{7.60}} \\

{Case 5}
& {11.61} & {0.1095} & {0.8714} & {0.00}
& {\textbf{29.14}} & {\textbf{0.8956}} & {\textbf{0.0995}} & {\underline{17.34}}
& {25.79} & {0.8122} & {0.1446} & {91.01}
& {\underline{27.87}} & {\underline{0.8704}} & {\underline{0.1188}} & {\textbf{7.53}} \\

\hline
{Avg.}
& {12.81} & {0.1536} & {0.8039} & {0.00}
& {\textbf{29.70}} & {\underline{0.8872}} & {\textbf{0.1020}} & {\underline{16.92}}
& {27.10} & {0.8396} & {0.1267} & {89.80}
& {\underline{29.68}} & {\textbf{0.8948}} & {\underline{0.1033}} & {\textbf{7.63}} \\
\hline
\end{tabular}}
\vspace{0.3em}
\parbox{0.82\columnwidth}{
\scriptsize
\textit{Note:} Flex-DLD is time-consuming due to its iterative nature. Three iterations are adopted as a practical compromise between performance and computational cost.}
\end{table}

\subsubsection{Comparison with Implicit Neural Representation Methods}
To further evaluate FRHD against recent self-supervised HSI denoising approaches, we consider two implicit neural representation (INR) based methods, HLRTF~\cite{luo2022hlrtf} and Flex-DLD~\cite{Chen2024flex}, on the WDC dataset under five representative noise settings. Both methods leverage INR to model the inherent spatial-spectral structure of HSIs: HLRTF uses a neural network to learn a low-rank tensor decomposition, while Flex-DLD focuses on the spectral low-rank structure through network-based learning. Table~\ref{inr_results} presents the quantitative results, facilitating a direct comparison of restoration accuracy and computational efficiency.

From Table~\ref{inr_results}, it can be observed that both INR-based methods improve PSNR and SSIM compared with the raw noisy inputs, indicating their capability to capture spatial-spectral continuity in HSI. HLRTF generally achieves competitive denoising accuracy with relatively low computational cost due to its low-rank tensor formulation, whereas Flex-DLD tends to require significantly more runtime despite moderate restoration performance. The proposed FRHD shows comparable or higher PSNR and SSIM than HLRTF and Flex-DLD in most cases, with SAM values remaining similar across the methods. For example, in Case 1, FRHD achieves a PSNR of 35.22 dB and an SSIM of 0.9522, while HLRTF achieves 31.43 dB and 0.9316, and Flex-DLD achieves 31.14 dB and 0.9158. In terms of computational time, FRHD requires 5.73 seconds compared with 17.98 seconds for HLRTF and 89.62 seconds for Flex-DLD. These results suggest that FRHD provides competitive restoration performance while maintaining relatively low computational cost.

Regarding potential integration of INR into FRHD, our current framework does not explicitly incorporate INR modules. Integrating INR presents several challenges: (i) training INR networks on high-dimensional HSIs is typically iterative and memory-intensive, which may undermine FRHD’s real-time capability; (ii) FRHD’s hierarchical low-rank modeling already implicitly captures key spectral correlations and non-local self-similarity, partially overlapping with the representational capacity of INR. Nonetheless, hybrid schemes that combine FRHD with INR-based continuous representations could be a promising direction for future research, potentially further enhancing spectral-spatial consistency and denoising performance.

Overall, these additional comparisons strengthen the comprehensiveness of our experimental evaluation and further demonstrate FRHD’s competitive advantage over emerging INR-based HSI denoising methods.

\subsection{Convergence analysis}   \label{Convergence_section}
In this subsection, we analyze the convergence of the ADMM algorithm for FRHD. As a prerequisite, we introduce the following Lemma 1 and Theorem 2. The proof of Theorem~\ref{thm2} is available in Ref.~\cite{he20121On} and omitted here for brevity, while the proof of Lemma~1 is provided in Appendix~B.

$\textit{\textbf{Lemma \hypertarget{Lemma}{1}:}}$ Consider the ADMM optimization variables $\mathbf{x} = [\mathbf{U}, \mathbf{W}, \mathbf{S}, \mathbf{D}]^\top$ and $\mathbf{z} = [\mathbf{H}_1, \mathbf{H}_2, \mathbf{Q}_1, \mathbf{Q}_2]^\top$, with auxiliary variables $\mathbf{Q}_1 = \mathbf{W} \odot (\mathbf{Y} - \mathbf{U} \mathbf{V}^\top)$ and $\mathbf{Q}_2 = \bm{\mathfrak{T}} - \mathbf{W}$. Then, in each ADMM block update, where a single block variable is optimized while all others are fixed, the feasible set of the updated block is nonempty, closed, and convex. 

\begin{figure}[!t]
\centering
\subfloat[]{\includegraphics[width=1.65in]{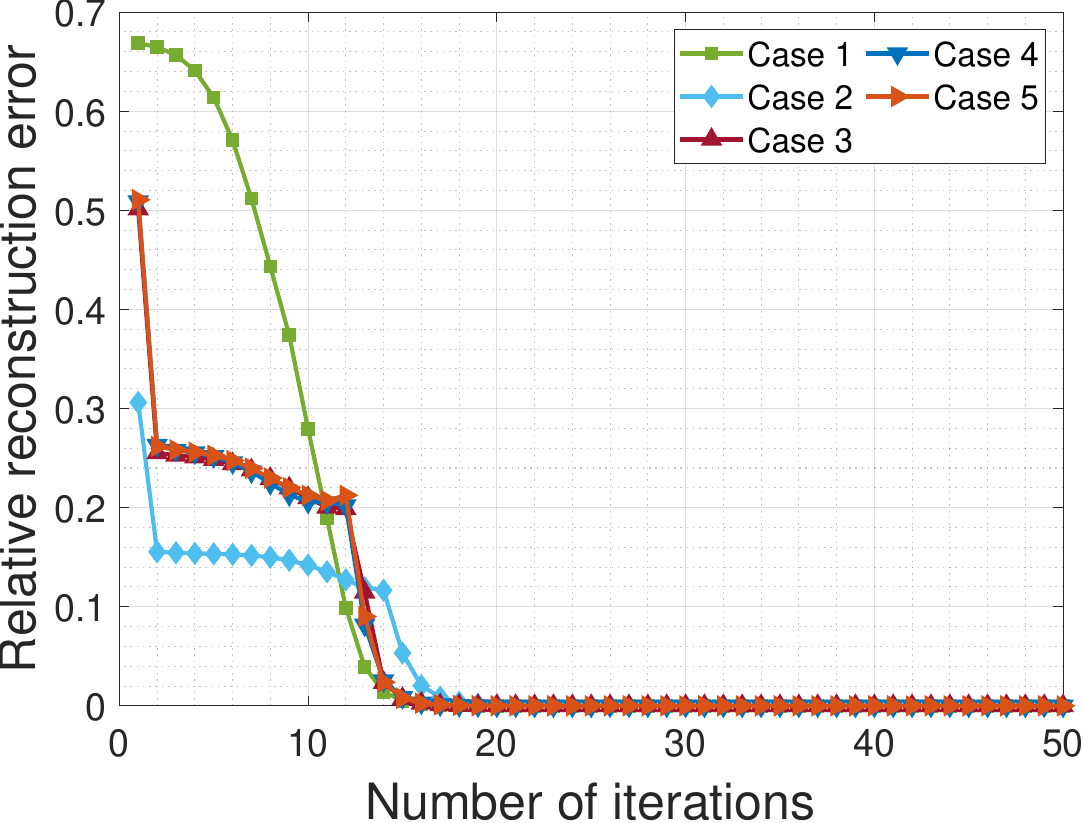}%
\label{CAVE_convergence_analysis}}
\hfil
\subfloat[]{\includegraphics[width=1.65in]{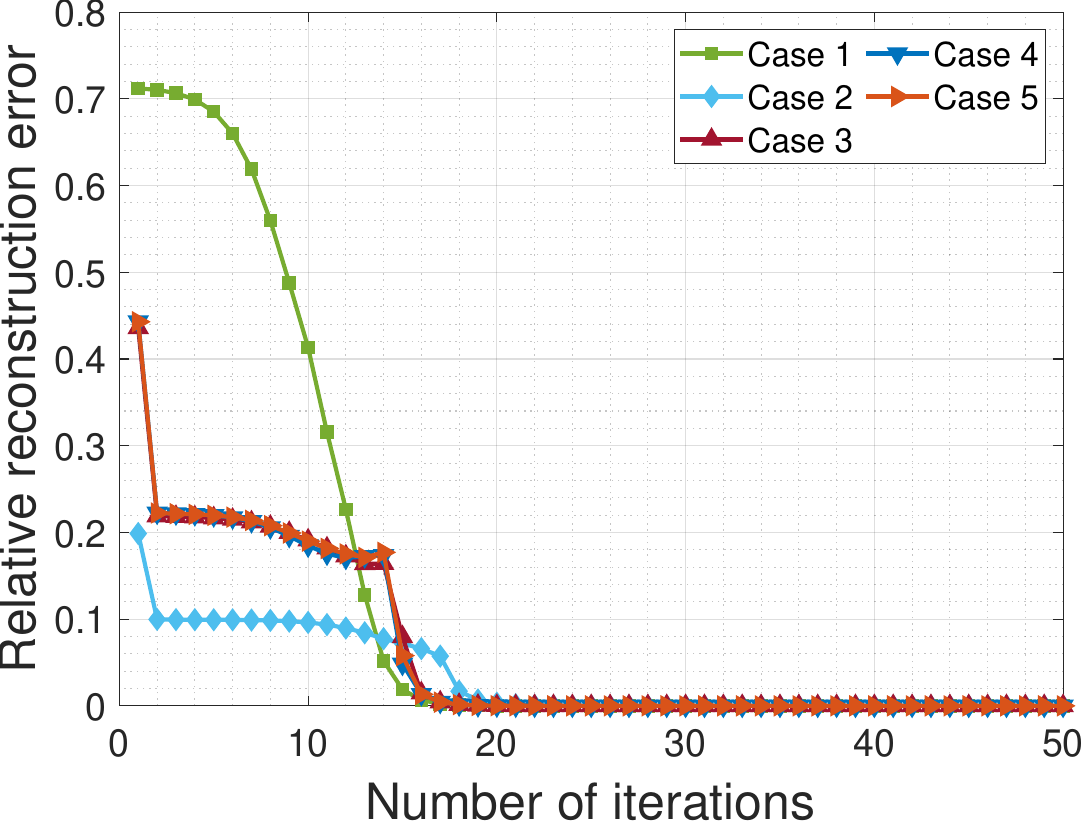}%
\label{PaC_convergence_analysis}}
\hfil
\subfloat[]{\includegraphics[width=1.65in]{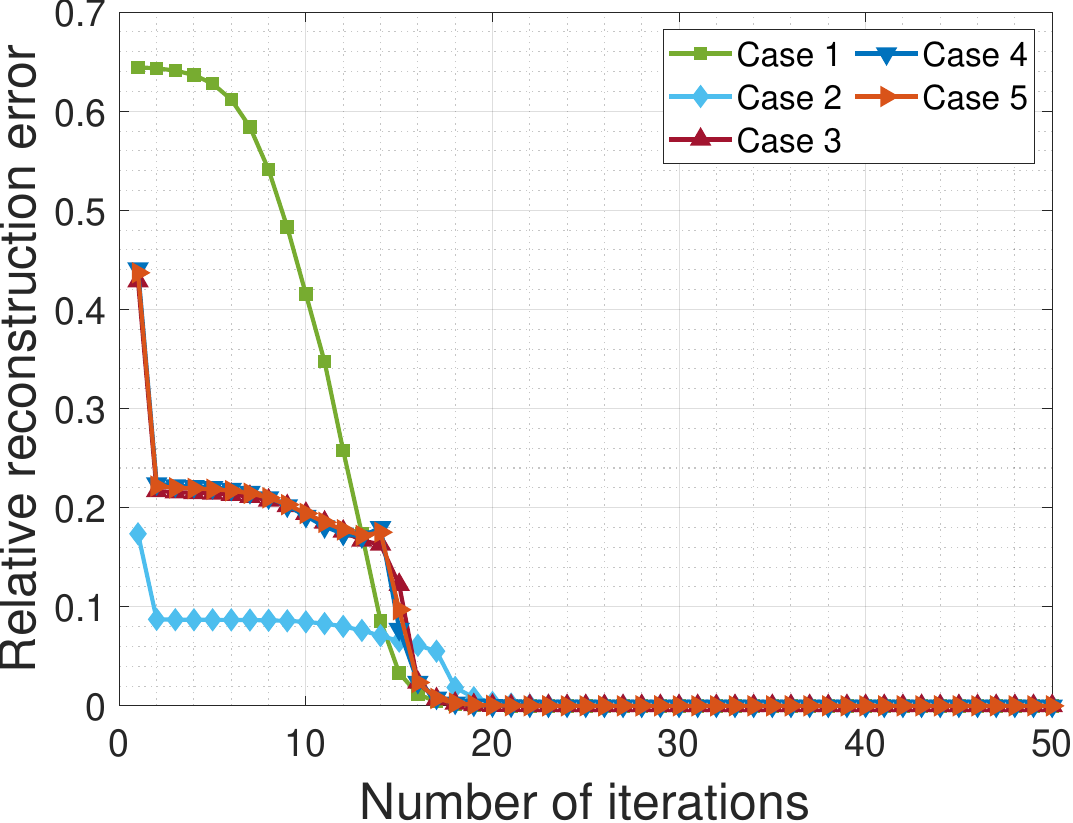}%
\label{WDC_convergence_analysis}}
\caption{Iteration results of FRHD method on different simulated datasets. (a) CAVE dataset; (b) PaC dataset; (c) WDC dataset.}
\label{Convergence_analysis}
\end{figure}

\newtheorem{theorem2}{\bf Theorem}
\begin{theorem}[\cite{he20121On}] \label{thm2}
The ADMM algorithm can be used to obtain the global optimal value of (\ref{admm}) at a rate of $\mathcal{O}(1/k)$.
\begin{align}   \label{admm}
    \min \ & f(\mathbf{x}) + g(\mathbf{z}), \notag \\
    \text{s.t.} \ & \mathbf{A} \mathbf{x} + \mathbf{B} \mathbf{z} = \mathbf{c}, \mathbf{x} \in \mathcal{X} \hspace{0.5em} \text{and} \hspace{0.5em}  \mathbf{z} \in \mathcal{Z},
\end{align}
where $f(\cdot)$ and $g(\cdot)$ are closed convex functions, $\mathbf{A} \in \mathbb{R}^{p \times n_1}$, $\mathbf{B} \in \mathbb{R}^{p \times n_2}$, $\mathbf{c} \in \mathbb{R}^p$, and $\mathcal{X,Z}$ are closed convex sets.
\end{theorem}

Although the FRHD model can be formally expressed in a form resembling Theorem~\ref{thm2} through certain transformations, the overall optimization problem is nonconvex due to the bilinear factorization \(\mathbf{X} = \mathbf{U}\mathbf{V}^\top\) and the orthogonality constraint \(\mathbf{V}^\top \mathbf{V} = \mathbf{I}\). Consequently, the convergence guarantee of Theorem~\ref{thm2} does not directly apply to the FRHD problem. Nevertheless, we establish the following result:

\newtheorem{theorem3}{\bf Theorem}
\begin{theorem}  \label{thm3}
The proposed ADMM algorithm converges to a stationary point of the FRHD optimization problem.
\end{theorem}

The proof is provided in Appendix~C. This result ensures that, although the problem is nonconvex, the algorithm is guaranteed to converge to points satisfying the first-order optimality conditions of the augmented Lagrangian.  Furthermore, our empirical convergence analysis (Fig.~\ref{Convergence_analysis}) demonstrates that the proposed method exhibits rapid and stable convergence across all tested datasets. Specifically, under various noise levels, the relative reconstruction error quickly stabilizes within approximately 20 iterations, indicating strong practical convergence behavior of the algorithm.

\section{Conclusion}  \label{sect6}
To address the limitations in balancing noise assumptions and fidelity terms, in this paper, we develop a framework incorporating noise prior reduction and spatial–spectral adaptive fidelity terms. By integrating the proposed noise prior reduction mechanism, pixel-wise weighting strategy, and additional RCTV regularization into a unified variational formulation, we further propose the FRHD model. This model effectively captures both spectral low-rank and local smoothness characteristics, while accommodating diverse noise constraints with varying properties. To ensure efficient and stable convergence, we leverage the ADMM for optimization. Extensive experiments on both simulated and real-world datasets demonstrate that the FRHD model achieves state-of-the-art denoising performance while maintaining competitive computational efficiency. The advantages of the FRHD model lie in rapid convergence, high accuracy, and robustness. It is also worth noting that, although FRHD remains a model-based approach, we encourage the application of learning-based methods within the proposed framework. Consequently, future work could focus on more detailed analysis within the proposed noise priors reduction denoising framework and explore the integration of the adaptive pixel-wise weighting strategy into deep learning methods to further enhance HSI reconstruction quality.

\section*{Conflict of interest}
The author declares that there are no conflicts of interest related to the content of this article.

\section*{CRediT authorship contribution statement}
\textbf{Xuelin Xie}: Conceptualization, Methodology, Software, Visualization, Writing - Original draft. \textbf{Xiliang Lu}: Conceptualization, Methodology, Supervision, Writing - review \& editing.
\textbf{Zhengshan Wang}: Validation, Writing - review \& editing.
\textbf{Yang Zhang}: Data curation, Writing - review \& editing.
\textbf{Long Chen}: Conceptualization, Methodology, Supervision, Writing - review \& editing.

\section*{Acknowledgements} 
We would like the thank the anonymous referees and associated editor for their useful comments and suggestions, which have led to considerable improvements in the paper. This work was supported in part by the Science and Technology Development Fund, Macao S.A.R. under Grant 0089/2024/AGJ and Grant 0012/2025/RIA1, in part by the University of Macau and the University of Macau Development Foundation under Grant MYRG-GRG2023-00106-FST-UMDF, in part by the National Natural Science Foundation of China under Grant 12371424, Grant U24A2002, and Grant 12561160122, and in part by the Natural Science Foundation of Hubei Province, China under Grant 2024AFE006.

\setcounter{equation}{0}
\renewcommand{\theequation}{A-\arabic{equation}}

\setcounter{figure}{0}
\renewcommand{\thefigure}{A-\arabic{figure}}

\setcounter{table}{0}
\renewcommand{\thetable}{A-\arabic{table}}

\section*{Appendix}
\subsection*{Appendix A. Proof of the Theorem 1.}
\begin{proof}
We begin by reviewing the core formulations of the initial Noise refinement HSI denoising framework and the Noise prior reduction HSI denoising framework, as defined in Eqs. (9) and (10) in our main manuscript.

Then, we may as well assume that the regularization of $\mathcal{X}$ is governed by a smooth F-norm, and employ the ADMM to solve the subproblems outlined in Eq. (9). In this case, the augmented Lagrangian function associated with Definition~1 is formulated as:
\begin{align}
\mathcal{L}{(\mathcal{X},\mathcal{G},\mathcal{S},\mathcal{D})}& =  \tau \left\| \mathcal{X} \right\|_{F}^{2}+\lambda \left\| \mathcal{G} \right\|_{F}^{2}+\beta {{\left\| \mathcal{S} \right\|}_{1}}+\gamma {{\left\| \mathcal{D} \right\|}_{2,1}}  +  \frac{\rho}{2} \left\| \mathcal{Y\!-\!X\!-\!G-\!S\!-\!D}+\! \frac{{\Lambda}}{\rho} \right\|_{F}^{2},
\end{align}
where $\rho$ is the proximal parameter, $\Lambda$ is the Lagrange multiplier, $\tau$, $\lambda$, $\beta$, and $\gamma$ are the regularization parameters.

The solution to the above subproblem that minimizes $\mathcal{G}$ is given by $\mathcal{G} = \frac{\rho}{2\lambda + \rho} \mathcal{A}$, where $\mathcal{A} = \mathcal{Y} - \mathcal{X} - \mathcal{S} - \mathcal{D} + \frac{\Lambda}{\rho}$. Substituting this expression into the objective function to eliminate the Gaussian noise $\mathcal{G}$, then get:
\begin{align}
\mathcal{L}{(\mathcal{X}, \mathcal{S},\mathcal{D})}& =  
 \tau \left\| \mathcal{X} \right\|_{F}^{2}+\beta {{\left\| \mathcal{S} \right\|}_{1}}+\gamma {{\left\| \mathcal{D} \right\|}_{2,1}}  + \frac{\rho}{2} \left( \frac{2\lambda}{2\lambda \!+\! \rho} \right) \! \left\| \mathcal{Y\!-\!X\!-\!S\!-\!D}+\! \frac{{\Lambda}}{\rho} \right\|_{F}^{2}.
\end{align}

This is also equivalent to:
\begin{align}  \label{problem1}
\mathcal{\tilde{L}}{(\mathcal{X}, \mathcal{S},\mathcal{D})}& =  
 \tilde{\tau} \left\| \mathcal{X} \right\|_{F}^{2}+ \tilde{\beta} {{\left\| \mathcal{S} \right\|}_{1}}+ \tilde{\gamma} {{\left\| \mathcal{D} \right\|}_{2,1}} + \frac{\rho}{2} \! \left\| \mathcal{Y\!-\!X\!-\!S\!-\!D}+\! \frac{{\Lambda}}{\rho} \right\|_{F}^{2},
\end{align}
where the parameters $\tilde{\tau}$, $\tilde{\beta}$, and $\tilde{\gamma}$ are:
\begin{align}
 \tilde{\tau }=  \left( 1+ \frac{\rho}{2\lambda} \right)\tau;
\tilde{\beta}= \left( 1+ \frac{\rho}{2\lambda} \right)\beta; \tilde{\gamma}= \left( 1+ \frac{\rho}{2\lambda} \right)\gamma.
\end{align}

In this case, we proceed by applying the ADMM framework to solve the problem, with the solutions to the other subproblems given by:
\begin{align}  \label{def1solve1}
\begin{cases}
 \mathcal{S}=\mathcal{S}_{\tilde{\beta}/\rho}(\mathcal{Y-X-D}+\frac{\Lambda }{\rho });  \\[4pt]
 \mathcal{D}= \text{fold3}\left( \text{max}(1- \tilde{\gamma}/ \! \left\| \mathbf{\Omega_j} \right\|_{F}^{2},0)\right);  \\[4pt]
 \mathcal{X}=\frac{\rho }{2 \tilde{\tau}+\rho }(\mathcal{Y-S-D}+\frac{\Lambda }{\rho }), \\[4pt]
\end{cases}
\end{align}
where $\mathcal{S}_{\tilde{\beta}/\rho}(\cdot)$ is the soft threshold operator, and $ \mathbf{\Omega_j}=\mathbf{\{Y-X-S\}(:,j)}+\frac{\mathbf{\Lambda(:,j)} }{\rho }$.

On the other hand, the augmented Lagrangian corresponding to the Noise prior reduction HSI denoising framework (Definition~2) is expressed as:
\begin{align}
\mathcal{L}{(\mathcal{X},\mathcal{S},\mathcal{D})}& = \tau_2 \left\| \mathcal{X} \right\|_{F}^{2}+\beta_2 {{\left\| \mathcal{S} \right\|}_{1}}+\gamma_2 {{\left\| \mathcal{D} \right\|}_{2,1}}  +  \frac{\rho}{2} \left\| \mathcal{Y\!-\!X\!-\!S\!-\!D}+\! \frac{{\Lambda}}{\rho} \right\|_{F}^{2}.
\end{align}

Obviously, this problem has the same form as problem (\ref{problem1}), and the solutions to the corresponding subproblems are:
\begin{align}  \label{def1solve2}
\begin{cases}
 \mathcal{S}=\mathcal{S}_{\beta_2/\rho}(\mathcal{Y-X-D}+\frac{\Lambda }{\rho });  \\[4pt]
 \mathcal{D}= \text{fold3}\left( \text{max}(1- \gamma_2/ \! \left\| \mathbf{\Omega_j} \right\|_{F}^{2},0)\right);  \\[4pt]
 \mathcal{X}=\frac{\rho }{2\tau_2+\rho }(\mathcal{Y-S-D}+\frac{\Lambda }{\rho }).  \\[4pt]
\end{cases}
\end{align}

Accordingly, it suffices to appropriately adjust the regularization parameters $\tau_2 = \tilde{\tau}$, $\beta_2 = \tilde{\beta}$, and $\gamma_2 = \tilde{\gamma}$ to ensure that the two problems yield identical solutions. This demonstrates the equivalence between the Noise refinement HSI denoising framework and the Noise prior reduction HSI denoising framework within the ADMM optimization paradigm. Thus, the proof is concluded.
\end{proof}

\section*{Appendix B. Proof of Lemma 1}
\begin{proof}
Consider a single block update in ADMM, where all other variables are fixed.  

First, for each independent variable, the feasible set defined by the norm balls 
$\mathcal{U}, \mathcal{W}, \mathcal{S}, \mathcal{D}, \mathcal{H}_1, \mathcal{H}_2$ 
is nonempty, closed, and convex.  

Then, for the auxiliary variable $\mathbf{Q}_2 = \mathfrak T - \mathbf{W}$, the mapping is affine in $(\mathbf{W}, \mathbf{Q}_2)$, and hence its feasible set is closed and convex.  

Regarding $\mathbf{Q}_1 = \mathbf{W} \odot (\mathbf{Y} - \mathbf{U} \mathbf{V}^\top)$, although the mapping $(\mathbf{U},\mathbf{W}) \mapsto \mathbf{W} \odot (\mathbf{Y} - \mathbf{U} \mathbf{V}^\top)$ is bilinear and nonconvex globally, it becomes affine with respect to the block variable being updated:  
\begin{itemize}
  \item When updating $\mathbf{U}$ with $\mathbf{W},\mathbf{V}$ fixed, the mapping $\mathbf{U} \mapsto \mathbf{W} \odot (\mathbf{U} \mathbf{V}^\top)$ is linear, hence affine.
  \item When updating $\mathbf{W}$ with $\mathbf{U},\mathbf{V}$ fixed, the mapping $\mathbf{W} \mapsto \mathbf{W} \odot (\mathbf{Y} - \mathbf{U} \mathbf{V}^\top)$ is linear, hence affine.
  \item When updating $\mathbf{Q}_1$, all $\mathbf{U},\mathbf{W},\mathbf{V}$ are fixed, so $\mathbf{Q}_1$ is constant and imposes no convexity issue.
\end{itemize}

Therefore, in each block update, the feasible set of the updated variable is the intersection of a norm ball and an affine set, which is itself closed and convex. Since the objective function in each block is continuous and convex in the updated variable, Weierstrass' theorem guarantees that a minimizer exists over this compact feasible set.
\end{proof}

\section*{Appendix C. Proof of Theorem 3}

\begin{proof}
We first note that problem~(\ref{Finall_lagrange}) is nonconvex due to the bilinear factorization $\mathbf{X}=\mathbf{U}\mathbf{V}^\top$ and the orthogonality constraint $\mathbf{V}^\top \mathbf{V}=\mathbf{I}$. Hence, classical convergence results for convex ADMM do not directly apply, and we analyze convergence within the framework of nonconvex block-coordinate ADMM.

The augmented Lagrangian is semi-algebraic, consisting of quadratic terms, the box constraint $\mathbf{W} \in [0,1]^{MN \times B}$ (encoded via an indicator function), and the Stiefel manifold constraint $\{\mathbf{V} \mid \mathbf{V}^\top \mathbf{V} = \mathbf{I}\}$. Semi-algebraic functions are closed under finite sums and compositions, and thus satisfy the Kurdyka-Łojasiewicz (KL) property.

Each subproblem is well-defined. The $\mathbf{U}$-subproblem is strictly convex with a unique minimizer. The $\mathbf{V}$-subproblem reduces to an orthogonal Procrustes problem, solved via SVD. The $\mathbf{W}$-subproblem is a convex quadratic program with box constraints, yielding a unique solution. The remaining subproblems for $\mathbf{H}_i$, $\mathbf{S}$, and $\mathbf{D}$ are proximal operators of convex functions and are uniquely defined. The feasible sets are compact: $\mathbf{W}$ is confined to $[0,1]^{MN \times B}$, and $\mathbf{V}$ to the Stiefel manifold. Hence, the generated sequence is bounded and possesses at least one accumulation point.

With the KL property, well-defined subproblems, and bounded iterates, the proposed algorithm falls within the standard analytical framework for nonconvex ADMM (e.g., \cite{hong2016convergence}). Consequently, every accumulation point of the generated sequence satisfies the Karush–Kuhn–Tucker conditions of the augmented Lagrangian; that is, it is a stationary point. This theoretical result is supported by our empirical observations in Fig.~\ref{Convergence_analysis}, where the relative reconstruction error stabilizes rapidly across diverse datasets and noise levels.

\end{proof}

\bibliographystyle{elsarticle-num}

\bibliography{Reference}


\end{document}